\documentclass[twoside]{article}

\usepackage{my_aistats_2024}

\begin{document}

\twocolumn[
  \aistatstitle{Provable Mutual Benefits from Federated Learning in
  Privacy-Sensitive Domains}
  \aistatsauthor{Nikita Tsoy \And Anna Mihalkova \And Teodora Todorova \And
  Nikola Konstantinov}
  \aistatsaddress{INSAIT, Sofia University\\ Sofia, Bulgaria \And ETH Zurich \\ Zurich, Switzerland \And High School of
  Mathematics\\ Burgas, Bulgaria \And INSAIT, Sofia University\\ Sofia, Bulgaria}]

\begin{abstract}
Cross-silo federated learning (FL) allows data owners to train accurate machine learning models by benefiting from each others private datasets. 
Unfortunately, the model accuracy benefits of collaboration are often undermined by privacy defenses. Therefore, to incentivize client participation in privacy-sensitive domains, a FL protocol should strike a delicate balance between privacy guarantees and end-model accuracy.
In this paper, we study the question of when and how a server could design a FL protocol provably beneficial for all participants. First, we provide necessary and sufficient conditions for the existence of mutually beneficial protocols in the context of mean estimation and convex stochastic optimization. We also derive protocols that maximize the total clients' utility, given symmetric privacy preferences. Finally, we design protocols maximizing end-model accuracy and demonstrate their benefits in synthetic experiments.
\end{abstract}

\section{INTRODUCTION}
The fundamental reliance of modern machine learning algorithms on diverse, high-quality
data has recently led to a surge of interest in collaborative learning
techniques. Among these, federated learning (FL) stands out because it enables
collaborative training in a fully distributed manner by exchanging gradient
updates via a central server \citep{m17c, k16f, k21a}. However, despite the distributed nature of these protocols, recent work has shown that the
communicated gradients are often enough to reconstruct sensitive information
about the participants \citep{g20i, h22v, z19d}.

These vulnerabilities emphasize the importance of privacy protection techniques
in FL, e.g., in terms of differential privacy (DP) \citep{d14a}. Such defenses
usually add a certain level of noise to the communicated messages, with higher
noise levels leading to stronger privacy guarantees. However, this noise often
harms the end-model performance, potentially undermining all benefits of
collaboration \citep{b14p, a16d}.

Such a delicate balance between model accuracy and data privacy brings into
question the viability of FL, especially for entities managing sensitive data.
While such entities can improve their services using FL, they also seek to provide
strong privacy guarantees to their clients. Therefore, they may be unwilling to
participate in a FL protocol that does not match their preferences in terms of
the accuracy-privacy trade-off.

\paragraph{Contributions} We address the question of whether and how FL can be
made \textit{provably mutually beneficial to all its participants}. To this
end, we consider a formal framework that models the clients' privacy and
accuracy preferences and the server's objectives. We provide
necessary and sufficient conditions for the existence of mutually beneficial FL
protocols in the context of two classic learning tasks: mean estimation and
strongly convex stochastic optimization. Our results cover both DP and a
privacy notion based on a local data reconstruction loss.

We also study the question of how the server can maximize two natural
objectives, the total client utility or the end-model accuracy, using mutually
beneficial protocols. For the first objective, in the case of symmetric client
utilities, we provide rich theoretical descriptions of the optimal noise levels
and the corresponding utility benefits from collaboration. For the second
objective, we provide simulations on synthetic data, which demonstrate the
benefits of tailoring the FL protocol to the client incentives.

\section{RELATED WORK}
\paragraph{Accuracy-Privacy Trade-off in FL} Privacy is a major concern in FL
\citep{k21a} and differential privacy mechanisms are popular tools for
addressing this problem \citep{a16d, p23d}. However, often privacy protection
decreases the resulting accuracy of the model. Thus, numerous works study the
accuracy-privacy trade-off. In particular, \citet{b19p} and \citet{f20p} study
the optimal rates for stochastic optimization given a fixed level of privacy
protection, and \citet{a22o} and \citet{c21c} derive the min-max optimal rates
for private mean estimation. 

Our work diverges from these studies in three significant ways. First, our work
incorporates the participation constraints of the clients
(\cref{sec:feasibility}). Second, we provide a theoretically grounded
methodology to pick up points from the accuracy-privacy frontier (Sections
\ref{sec:framework}, \ref{sec:util-max-sym}, and \ref{sec:acc-max}). Third, our
approach accounts for differences in privacy preferences and allows the server
to personalize privacy protections as long as all clients benefit in utility
(\cref{sec:framework}).

\paragraph{Privacy-Related Incentives in FL} Several recent works consider
privacy-related incentives in federated learning. In particular, \citet{s23s}
consider how to move closer to the accuracy-privacy Pareto front under
statistical client heterogeneity. However, they do not model heterogeneity in
the client privacy preferences and do not provide tools to find the right level
of privacy protection. \citet{h22i} study ways of incentivizing clients to
contribute data to an FL protocol. However, they do not study how to construct
an optimal protocol for the accuracy-privacy trade-off, focusing on
incentivization schemes given a fixed learning protocol. Similarly to our work,
\citet{g11s} and \citet{a23s} also study how to construct a privacy protocol
beneficial for the server given participation constraints. However, they only
consider monetary payments as a way to compensate the clients, while we
consider compensations based on end-model accuracy, which seems more fitting
for the case of collaborative learning.

\paragraph{Other Incentives in Federated Learning} Other incentives, apart from
privacy, have also been considered in the context of FL. In particular, recent
work studies incentives to free-ride and corresponding ways to design fair
payment allocation schemes for the participants \citep{l20c, d21, b21o, k22m}.
Another lines of work studied incentives stemming from downstream competition
\citep{t24s, d23i}, as well as ML games for delegated data collection
\cite{s24d, a23d} and fine-tuning foundation models \cite{l23f}.

Also related is the field of data valuation \citep{w20p,p20s,r22s}. A key
conceptual difference to our work is that such studies seek to quantify the
usefulness of data relative to the other players, e.g., by evaluating the
amount of informativeness or redundancy of a data source given the others’
data. In contrast, we study the accuracy improvements with all clients’ data
and how it is affected by privacy defenses.

\section{GENERAL FRAMEWORK}
\label{sec:framework}

Our work analyzes a fundamental accuracy-privacy trade-off in federated
learning from a multi-objective perspective. We consider a standard federated
learning setup, where $N$ potential participants with local datasets
$(D_i)_{i=1}^N$ may train a model together. While collaboration may improve
end-model accuracy, it also increases privacy risk. Therefore, the clients will
join the protocol if and only if they expect a substantial increase in model's
accuracy from collaboration. To study this trade-off, we propose a general
framework that consists of four parts: the federated learning protocol, the
clients' evaluations of the protocol, the clients' utility functions, and the
server's objective.

\subsection{Federated Learning Protocol}

FL protocol $p$ describes all details of the server-client interactions,
including the distributed learning algorithm and privacy protection mechanism.
It details all training parameters, such as optimization step count, batch
size, and learning rate, and determines (possibly client-specific) privacy
protection parameters, such as privacy-preserving noise levels, $\vec{\alpha} =
(\alpha_1, \alpha_2, \ldots, \alpha_N)$. Our work focuses on adjusting
privacy-preserving noise levels, $\vec{\alpha}$, since they are the most
important in the accuracy-privacy trade-off. Given the noise levels, we
optimize the remaining parameters for the client utilities. We denote the
resulting protocol as $p_{\vec{\alpha}}$.

\subsection{Clients' Evaluations}

We assume that the clients could reason about the proposed protocol by
estimating the end model error and training procedure privacy risk. Formally,
we introduce two sets of functions: $(\err_i)_{i=1}^N$ and
$(\leak_i)_{i=1}^N$.

\paragraph{$\err_i\colon P \to [0, \infty)$} describes client $i$'s
evaluation of the end-model error given protocol $p$. Since final error is
unknown before training, the clients may use theoretical guarantees to reason
about it, such as \emph{expected} values or \emph{high probability upper
bounds}. For example, in the case of collaborative mean estimation, the clients
might use the expected mean squared error as an $\err$ function.

\paragraph{$\leak_i\colon P \to [0, \infty)$} describes client $i$'s
evaluation of the privacy violation given protocol $p$. Since privacy leaks are
hard to predict, the client will again reason via theoretical guarantees. For
example, the clients might quantify privacy protection using the popular notion
of \emph{differential privacy} \citep{d14a}.

\begin{definition}
  \label{defn:dp}
  A randomized algorithm $A\colon \mathcal{D} \to \mathcal{S}$ is
  $(\varepsilon, \delta)$-differentially private if, for all neighbouring
  datasets $D$ and $D'$ (i.e. datasets differing in one point),
  \begin{equation*}
    \forall S \subseteq \mathcal{S} \: \Pr(A(D) \in S) \le \me^\varepsilon
    \Pr(A(D') \in S) + \delta.
  \end{equation*}
\end{definition}

DP requires that the outputs of an algorithm are almost indistinguishable,
regardless of whether a certain individual data point included or not into the
dataset. Intuitively, $\epsilon$ measures how many bits of information the
algorithm's outputs reveal about the input, while $\delta$ represents the
probability of such guarantees failing to hold. Therefore, it is a common
practice to seek $\delta \ll \nicefrac{1}{\abs{D}}$ \citep{p23d}. In our work,
we consider algorithms that ensure $\delta = \O(\nicefrac{1}{\abs{D}^2})$.

\subsection{Clients' Utilities and Rationality}

We further assume that the clients use the evaluations from the previous
section to reason about the benefits of a protocol. Formally, client $i$ gets
utility $u_i(\err, \leak)$ from a protocol with estimated accuracy loss $\err$
and privacy loss $\leak$. To derive quantitative results about the
accuracy-privacy trade-off, we introduce a specific form of this utility
function. We also present results about general utility functions in Section
\ref{sec:general_existence}.

\paragraph{Parametric Utility Form} In our quantitative analysis, we assume the
following form of utility:
\[
  %\label{eqn:linear_utility}
  u_i(\err, \leak) = u^{p, \lambda_i}_i(\err, \leak) \defeq -\leak^p -
  \lambda_i \err^p,
\]
which corresponds to a negative weighted $l_p$-norm of the vector $(\err,
\leak)$. Here, parameter $\lambda_i$ measures the importance of accuracy for
participant $i$ and is responsible for the correct rescaling of $\err$ and
$\leak$ units. Notice that this utility is monotonically decreasing and
continuous in both $\err$ and $\leak$, which reflects the preferences of the
clients towards better accuracy and privacy guarantees.

To motivate this parametric form, we use the classic result from the
multi-objective optimization literature. Since the clients benefit from the
protocols with low accuracy error and privacy leakage, we can treat the
minimization of these objectives as a multi-objective optimization problem. To
describe its Pareto front, one may analyze the linear combinations of the
objectives, $F^\text{lin} \defeq \sum_{i=1}^N w_i^{\leak} \leak_i^p +
w_i^{\err} \err_i^p$, since the minimizers of $F^\text{lin}$ lie on the Pareto
front \citep{e18t}. However, since
\begin{multline*}
  F^\text{lin} = \sum_{i=1}^N w_i^{\leak} \Par*{\leak_i^p +
  \lambda_i \err_i^p} = \\
  \sum_{i=1}^N -w_i^{\leak} u_i^{p, \lambda_i}(\err_i,
  \leak_i), \text{ where } \lambda_i \defeq \frac{w_i^{\err}}{w_i^{\leak}},
\end{multline*}
the problem of finding a Pareto frontier protocol will naturally lead to the
problem of optimizing the utility functions of our parametric family.

In our work, we will focus on the family $u^{2, \lambda_i}_i$, as one the most
natural choices among $l_p$-norms. To evaluate the robustness of our results to
the choice of utility, we also discuss a different utility function in
Appendix \ref{sec:other_utilities}.

\paragraph{Participation Constraints}
Based on the privacy, accuracy and utility analysis of the protocol, the clients can reason about the benefitis of joining it, compared to the an outside options for training (such as training a model alone). This creates contraints on what protocols are feasible in the first place. In general, additional contraints may reflect not only the individual desires of clients, but also possible legal or fairness requirements imposed by other entities (e.g. governments or the server itself).

In this work, we focus on constraints coming from the rationality of the clients. We therefore require protocols to be \emph{mutually beneficial} as a natural notion of feasibility.
\begin{definition}
\label{defn:mutually_beneficial}
  A protocol $p_{\alpha}$ is called mutually beneficial if for any
  player $i$, $u_i(\err_i(p), \leak_i(p)) \ge u^0_i$, where $u^0_i$ is the
  utility from local training.
\end{definition}
Mutually beneficial protocol ensures that all clients will benefit from
collaborative learning compared to individual learning. Therefore, it is rational for them to participate. We usually consider
$u^0_i = u^i(\err_i^0, 0)$, where $\err_i^0$ is the accuracy loss of local
model. We study feasibility of this constraint in \cref{sec:feasibility}.

\subsection{Server's Objectives}

Since there might be several protocols that will satisfy participation
constraints, the server may seek additional desirable properties. In our paper,
we will study two possible objectives of the server. The first one is
\textit{maximization of the total utility} of all participants
(\cref{sec:util-max-sym}). We study this objective for the case of symmetric
client utilities and provide a precise theoretical description the optimal
protocols. The second objective is the \textit{maximization of the server model
accuracy} (\cref{sec:acc-max}). We study this objective empirically and show
how personalization of privacy protection can improve the accuracy of the
global model.

\section{FEASIBILITY OF COLLABORATION}
\label{sec:feasibility}

We now focus on the existence of mutually-beneficial protocols. We
quantitatively analyze two canonical learning tasks: mean estimation (Section
\ref{sec:dp-mean-setting}) and strongly convex stochastic optimization (Section
\ref{sec:dp-sgd-existance}). First, we use DP as a privacy notion in both
setups. To evaluate the robustness of our findings to the choice of privacy
notion, we study also mean estimation with a privacy notion based on
reconstruction loss in Section \ref{sec:b-mean-setting}. Finally, we generalize
our findings in the limit case as the number of clients grows to infinity to
general utility functions in Section \ref{sec:general_existence}.

\subsection{Feasibility of DP Mean Estimation}
\label{sec:dp-mean-setting}

\paragraph{Setup} We begin the quantitative exploration of our framework with
the problem of DP mean estimation.

Formally, we assume that the clients are interested in parameter $\mu$ and
have noisy observations of this parameter. Entity $i$'s noisy observations of
this parameter, $(x^j_i)_{j=1}^n$, satisfy
\[
  \E(x^j_i) = \mu, \: \Var(x^j_i) = \sigma^2, \: \supp x^j_i \subseteq \mu +
  \Br*{-B/2, B/2}.
\]
Let $\bar{x}_i \defeq \nicefrac{\sum_{j=1}^n x^j_i}{n}$ be a local average of
all samples. The server will choose among the protocols $p_{\vec{\alpha}}$ in
which entity $i$ reveals the following noisy message to other participants $m_i
= \bar{x}_i + \epsilon_i$, where $\epsilon_i \sim \N(0, \alpha_i^2)$.

\paragraph{Accuracy Loss} The entities' accuracy loss will be the root of the
mean squared error
\[
  \err_i = \sqrt{\E_{\{D_i\}_{i=1}^N, \{\epsilon_i\}_{i=1}^N}((\hat{\mu}_i -
  \mu)^2)},
\]
where $\hat{\mu}_i$ is the best unbiased linear predictor of $\mu$ that can be
constructed from $D_i \cup \{m_1, \dots, m_N\}$. 
\begin{theorem}[Proof in \cref{sec:dp-mean-predictor-proof}]
  \label{thm:dp-mean-predictor}
  $\hat{\mu}_i$ has the following property
  \[
    \E((\hat{\mu}_i - \mu)^2) = \frac{1}{\gamma_i + \rho},
  \]
  where $\rho \defeq \frac{n}{\sigma^2}$, $\beta_i \defeq
  \frac{1}{\frac{1}{\rho} + \alpha_i^2}$, $\gamma_i \defeq
  \sum_{k \neq i} \beta_k$.
\end{theorem}
As we can see, the accuracy for client $i$ will depend on the
``informativeness'' of the messages of other clients, $\beta_k, k \neq i$.
``Informativeness'' $\beta_k$ is decreasing in noise $\alpha_k$ and is bounded
by the ``informativeness'' of the noiseless message, $\rho$.

\paragraph{Privacy Loss} We will use the standard result for DP of the Gaussian
mechanism \citep[][see Theorem \ref{thm:dp-mean-privacy} in
\cref{sec:background}]{d14a}, which gives the following bound for DP budget:
\[
  \varepsilon_i \le \frac{\kappa}{\alpha_i}, \text{ where } \kappa =
  \frac{\sqrt{2 \ln(1.25 n^2)} B}{n}.
\]
Here, $\kappa$ describes the strength of privacy concerns. This strength is
small when the observations of each client are similar (small $B$) or when
their message is close to the true mean of the distribution (big $n$).

The entities will use this upper bound to reason about their privacy loss,
$\leak_i = \nicefrac{\kappa}{\alpha_i}$. We assume that the clients will use
this upper bound even if $\varepsilon_i \ge 1$ and \cref{thm:dp-mean-privacy}
does not hold.

\paragraph{Participation Incentives} These choices result in utility function
\[
  u_i = -\leak_i^2 - \lambda_i \err_i^2 = -\frac{\kappa^2}{\alpha_i^2} -
  \frac{\lambda_i}{\gamma_i + \rho}.
\]
To find the utility of training alone, we consider limit $\forall k \:
\alpha_k \to \infty$ when clients use only their local data. This limit gives
$u^0_i = -\nicefrac{\lambda_i}{\rho}$. Therefore, the protocol is mutually
beneficial when
\begin{equation}
  \label{eq:dp-mean-gen-system}
  \forall i \: u_i \ge u^0_i \iff -\frac{\kappa^2 \rho
  \beta_i}{\rho - \beta_i} + \frac{\lambda_i \gamma_i}{\rho (\gamma_i + \rho)}
  \ge 0.
\end{equation}

Now, we provide a necessary and sufficient condition for the existence of
parameters $\beta_k$ (and corresponding $\alpha_k$) that yield a mutually
beneficial protocol.
\begin{theorem}[Proof in \cref{sec:dp-mean-existence-proof}]
  \label{thm:dp-mean-existence}
  \cref{eq:dp-mean-gen-system} has solutions if and only if
  \[
    \sum_{i=1}^N \zeta_i > 1, \text{ where } \zeta_i \defeq
    \frac{\lambda_i}{\lambda_i + \kappa^2 \rho^2}.
  \]

  Moreover, if the system has solutions, it has a series of solutions in
  this form,
  \[
    \beta_i = \zeta_i b, b \in \Br[\bigg]{0, \Par[\bigg]{1 -
    \frac{1}{\sum_{i=1}^N \zeta_i}} \rho}.
  \]
\end{theorem}

To interpret these results, we make two key observations. First,
\cref{eq:dp-mean-gen-system} implies $\rho \alpha_i^2 \ge \frac{\kappa^2
\rho^2}{\lambda_i}$. Thus, ratio $\nicefrac{\rho^2 \kappa^2}{\lambda_i}$
describes the minimally acceptable level of DP noise for client $i$. This level
of noise ``minimally'' trades off the privacy concerns (described by $\kappa$)
for accuracy gains (bounded by $\nicefrac{1}{\rho}$). The existence condition
in Theorem \ref{thm:dp-mean-existence} requires that preference towards
accuracy $\lambda_i$ is sufficiently large so that this minimal noise level
$\nicefrac{\rho^2 \kappa^2}{\lambda_i}$ is small enough to allow accuracy
benefits for everyone. Second, if all other parameters are fixed and
$\lambda_i$ are bounded from below, the sum in Theorem
\ref{thm:dp-mean-existence} diverges: collaboration becomes mutually
beneficial. This property shows that the accuracy benefits will outweigh the
potential privacy concerns as the number of players grows.

\subsection{Feasibility of DP Stochastic Optimization}
\label{sec:dp-sgd-existance}

\paragraph{Setup} Now, we analyze DP stochastic optimization.

Formally, we assume that the clients are interested in parameters $\w$ that
minimize strongly convex objective function $f$ on the closed and convex
set $W$, such that $f^* \defeq \min_{\w \in \mathbb{R}^d} f(\w) = \min_{\w \in
W} f(\w)$ and $\diam(W) \le D$. Each entity $i$ has access to gradient oracle
$\vec{g}_i$, which allows it to use its data points to get the unbiased
estimate of the gradient of the objective
\begin{align*}
  & \forall \w \: \nabla f(\w) = \E_{x^j_i}(\vec{g}_i(\w, x^j_i)),\\
  & \E(\norm{\nabla f(\w) - \vec{g}_i(\w, x^j_i)}^2) \le \sigma^2,\\
  & \forall \w \: \supp \vec{g}_i(\w, \cdot) \subseteq \nabla f(\w) +
  \mathrm{B}\Par*{B / 2},
\end{align*}
where $\mathrm{B}(R) \defeq \Bc*{\vec{y} \given \norm{\vec{y}} \le R}$. We
assume that the objective function is $L$-smooth on $\mathbb{R}^d$ and
$\mu$-strongly convex on $W$ (see \cref{sec:background} for the definitions of
smoothness and strong convexity).

The server will communicate with the clients for $m$ rounds. At each round, the
server gets the batch mean of local gradient estimates of the clients. The
clients use $b = \Floor*{\nicefrac{n}{m}}$ points per batch selected
randomly without replacement. The server then optimally aggregates
the gradient estimates and updates the global parameters
(\cref{alg:dp-sgd-algorithm}). In this algorithm, we do not consider the
additional randomization over the choice of clients because we assume that the
server will be able to determine who participated in each round of
computation. Moreover, we will pass through the data only once because even one
pass through the data guarantees the optimal asymptotic behavior of the
training objective.

\begin{algorithm}
  \caption{Collaborative Learning Protocol}
  \label{alg:dp-sgd-algorithm}
  \begin{algorithmic}
    \STATE {\bfseries Input:} protocol $p_{\vec{\alpha}}$
    \STATE Server randomly initializes $\w^0 \in W$
    \FOR{$i=1$ {\bfseries to} $N$}
      \STATE Client $i$ randomly chooses $\pi_i\colon [m] \to [m]$
    \ENDFOR
    \FOR{$t=1$ {\bfseries to} $m$}
      \FOR{$i=1$ {\bfseries to} $N$}
        \STATE Client $i$ samples $\vec{\xi}^t_i \sim \N(\vec{0},
        \alpha_i^2 \I)$
        \STATE Client $i$ calculates \[
          \vec{m}^t_i = \vec{\xi}^t_i + \frac{1}{b} \sum_{j=1}^b
          \vec{g}_i\Par*{\w^{t-1}, x^{(\pi_i(t) - 1) b + j}_i}
        \]
        \STATE Client $i$ sends $\vec{m}^t_i$ to the server
      \ENDFOR
      \STATE Server computes $\vec{g}^t = \sum_{i=1}^N a_i \vec{m}^t_i$
      \STATE Server updates $\vec{w}^t = \mat{\varPi}_W (\vec{w}^{t-1} -
      \eta^t \vec{g}^t)$
    \ENDFOR
  \end{algorithmic}
\end{algorithm}

\paragraph{Accuracy Loss} For fixed noise levels $\vec{\alpha}$, one can show
the following upper bound on the distance to optimal parameter $\w^*$. Here,
aggregation weights $(a^*_i)_{i=1}^N$ and step sizes $(\eta^t_*)_{t=1}^m$ are
chosen to minimize the bound.

\begin{theorem}[Proof in \cref{sec:dp-sgd-final-acc-proof}]
  \label{thm:dp-sgd-final-acc}
  Optimization outcome $\Delta \w^m \defeq \E(\norm{\w^m - \w^*}^2)$
  satisfies
  \[
    \Delta \w^m \le
    \begin{cases}
      \frac{1}{\Par*{1 + \frac{\chi}{2 - \chi} (m - T)} L \mu \varGamma}, &
      m \ge T,\\
      \Par*{1 - \frac{\chi}{2}}^m \frac{y^0_*}{L \mu \varGamma}, & m < T,
    \end{cases}
    \eqdef \Delta \w^{m, \text{ub}},
  \]
  where $T \defeq \Ceil[\Big]{-\frac{\ln(y^0_*)}{\ln\Par*{1 -
  \frac{\chi}{2}}}}$, $\chi \defeq \frac{\mu}{L}$, $y^0_* \defeq \frac{L \mu n
  N D^2}{\sigma^2}$, $\varGamma \defeq \sum_{i=1}^N \beta_i$, $\beta_i =
  \frac{1}{\nicefrac{1}{\rho} + d \alpha_i^2}$, $\rho \defeq
  \nicefrac{b}{\sigma^2}$.
\end{theorem}
As we can see, we get a similar dependence to \cref{thm:dp-mean-predictor}.
However, in this case, the optimization outcome for client $i$ will
also deteriorate in their noise contribution $\alpha_i$. Similarly to the
previous case, we assume $\err_i = \sqrt{\Delta \w^{m, \text{ub}}}$.

\paragraph{Privacy Loss} We use the result of \citet{f22h}, applied to the
standard Gaussian mechanism \citep{d14a}.

\begin{theorem}[Proof in \cref{sec:dp-sgd-privacy-proof}]
  \label{thm:dp-sgd-privacy}
  If $m \ge 2$ and the server uses $\cref{alg:dp-sgd-algorithm}$ and chooses
  $\varepsilon_{0, i} \in \Par[\big]{0, \min\Par[\big]{1,
  \ln\Par[\big]{\frac{n}{16 \ln(\nicefrac{2}{\delta})}}}}$, $\delta_0 =
  \nicefrac{1}{m n^2}$, and $\delta = \nicefrac{1}{n^2}$, the client will have
  the local $\Par[\big]{\varepsilon_i, \frac{1 + (\me^{\varepsilon_i} + 1) (1 +
  \me^{-\varepsilon_{0, i} / 2})}{n^2}}$-DP guarantee, where $\varepsilon_i$
  has the following property:
  \[
    \label{eq:dp-sgd-final-privacy}
    \varepsilon_i \le \frac{16 \sqrt{2 \me \ln(1.25 m n^2) \ln(4 n^2)}
    B}{\sqrt{m} b \alpha_i} \eqdef \frac{\kappa}{\alpha_i}.
  \]
\end{theorem}
Here, $\kappa$ behaves similarly to the mean estimation case but has an
additional multiplier, $\nicefrac{1}{\sqrt{m}}$, which comes from the
shuffling of the data. Similarly to the previous section, we assume the clients
use this bound to reason about privacy leak $\leak_i =
\nicefrac{\kappa}{\alpha_i}$. (Here, we will again assume that the clients will
ignore the restrictions on $\varepsilon_{0, i}$ and $m$.)

\paragraph{Participation Incentives} To choose the optimal batch size $b$ for
training, we consider the behavior of $\err_i$ and $\leak_i$ when $m \to
\infty$, $b \to \infty$, and $\alpha_i \approx 0$. In this limit, we get
\[
  \err_i^2 \propto \frac{1}{m \varGamma} \propto \frac{1}{n N}, \:
  \leak_i^2 \propto \frac{\ln(n)^2}{b^2 m}
\]
We see that the increase in batch size does not greatly affect the accuracy of
estimates but strongly decreases the privacy leak. Thus, our analysis considers
$m = T$, resulting in utility function
\[
  u_i = -\frac{\kappa^2}{\alpha_i^2} - \frac{\lambda_i}{L \mu \varGamma},
\]
where $\kappa \defeq \frac{16 \sqrt{2 e \ln(1.25 T n^2) \ln(4 n^2) T} B}{n}$.
We also ignore that $T$ is an integer and assume that $T =
\max\Par[\Big]{-\frac{\ln(y^0_*)}{\ln\Par*{1 - \frac{\chi}{2}}}, 1}$, $y^0_* =
\frac{L \mu n N D^2}{\sigma^2}$, $b = \frac{n}{T}$.

To model the utility of local training, $u^0_i$, we consider limit
$\alpha_i \to 0$ and $\forall k \neq i \: \alpha_k \to \infty$ for the accuracy
error and $\leak_i \to 0$ for privacy leakage. This limit gives $u^0_i =
-\nicefrac{\lambda_i}{L \mu \rho}$. Therefore, a protocol is mutually
beneficial when
\begin{equation}
  \label{eq:dp-sgd-gen-system}
  \forall i \: u_i \ge u^0_i \iff \psi_i \Par*{1 - \frac{\rho}{\varGamma}}
  \ge \frac{\beta_i}{\rho - \beta_i},
\end{equation}
where $\psi_i \defeq \frac{\lambda_i}{L \mu d \kappa^2 \rho^2}$. The
corresponding necessary and sufficient condition for the existence of a
mutually beneficial protocol is the following.

\begin{theorem}[Proof in \cref{sec:dp-sgd-existence-proof}]
  \label{thm:dp-sgd-existence}
  \cref{eq:dp-sgd-gen-system} has a solution if and only if the following
  inequality has a solution
  \[
    \sum_{i=1}^N \frac{\psi_i x}{(\psi_i + 1) x + 1} \ge x + 1, \: x \ge 0.
  \]
\end{theorem}

\begin{corollary}[Proof in \cref{sec:dp-sgd-existence-simple-proof}]
  \label{cor:dp-sgd-existence-simple}
  \cref{eq:dp-sgd-gen-system} has a solution if
  $\sum_{i=1}^N \frac{\psi_i}{\psi_i + 2} \ge 2$ and only if $\sum_{i=1}^N
  \frac{\psi_i}{\sqrt{\psi_i + 1}} \ge 4$.
\end{corollary}

Similarly to \cref{thm:dp-mean-existence}, $\nicefrac{1}{\psi_i}$ describes
the minimal level of noise for our problem. To make collaboration
beneficial for everyone, we need these minimal noise levels to be small.

\label{sec:sgd_existence}

\subsection{Feasibility of Bayesian Mean Estimation}
\label{sec:b-mean-setting}

\paragraph{Setup} To evaluate the robustness of our findings to the choice of
privacy notion, we also consider mean estimation with a different privacy loss.

Formally, we assume that the clients are interested in a parameter $\mu$
sampled from a prior distribution $\N\Par{0, \frac{1}{\tau}}$. Entity $i$'s
noisy observations of this parameter $(x^j_i)_{j=1}^n$ are sampled from
$\N(\mu, \sigma^2)$. Let $\bar{x}_i \defeq \frac{1}{n} \sum_{j=1}^n x^j_i$ be a
local average of all local samples. The server will use protocols
$p_{\vec{\alpha}}$, where entity $i$ reveals noisy message $m_i \defeq
\bar{x}_i + \epsilon_i$, where $\epsilon_i \sim \N(0, \alpha_i^2)$, to others.

\paragraph{Accuracy Loss} Similarly to the previous sections, the entities'
accuracy loss will be the root of the mean squared error
\[
  \err_i = \sqrt{\E_{\{D_i\}_{i=1}^N, \{\epsilon_i\}_{i=1}^N, \mu}((\hat{\mu}_i
  - \mu)^2)},
\]
where $\hat{\mu}_i$ is the best predictor of $\mu$ that can be constructed from
$D_i \cup \{m_1, \dots, m_N\}$.
\begin{theorem}[Proof in \cref{sec:b-mean-accuracy-proof}]
  \label{thm:b-mean-accuracy}
  $\hat{\mu}_i$ satisfies
  \[
    \E((\hat{\mu}_i - \mu)^2) = \frac{1}{\gamma_i + \rho + \tau},
  \]
  where $\rho \defeq \frac{n}{\sigma^2}$, $\beta_i \defeq
  \frac{1}{\frac{1}{\rho} + \alpha_i^2} \in [0, \rho]$, $\gamma_i \defeq
  \sum_{k \neq i} \beta_k$.
\end{theorem}
Similarly to \cref{thm:dp-mean-predictor}, the quality of the estimate for
client $i$ depends on the ``informativeness'' of the messages of the others.

\paragraph{Privacy Loss} The entities' privacy loss will be the root of the
negative mean squared error
\[
  \leak_i = \sqrt{S - \E_{\{D_i\}_{i=1}^N, \{\epsilon_i\}_{i=1}^N,
  \mu}\Par[\Big]{\frac{1}{n} \sum_{j=1}^n (\hat{x}^j_i - x^j_i)^2}},
\]
where $S \defeq \E_{\{D_i\}_{i=1}^N, \{\epsilon_i\}_{i=1}^N,
\mu}\Par[\big]{\frac{1}{n} \sum_{j=1}^n (x^j_i)^2} = \sigma^2 +
\frac{1}{\tau}$ is the data reconstruction error if the server uses a zero
estimator for local data and $\hat{x}^j_i$ is the best predictor of $x_i^j$
that can be constructed from $\{m_1, \dots, m_N\}$. 

Intuitively, this notion of privacy violation measures how well the private
data points can be reconstructed from the revealed messages relative to a
simple Bayesian estimate. We have the following formula for the amount of
leakage given optimal reconstruction.
\begin{theorem}[Proof in \cref{sec:b-mean-privacy-proof}]
  \label{thm:b-mean-privacy}
  $\{\hat{x}^j_i\}_{j=1}^n$ have the following property
  \[
    \E\Par[\Big]{\sum_{j=1}^n (\hat{x}^j_i - x^j_i)^2} =
    \sigma^2 (n - 1) + n \Par[\Big]{\frac{1}{\alpha_i^2} + \frac{\rho
    (\gamma_i + \tau)}{\gamma_i + \rho + \tau}}^{-1}
  \]
\end{theorem}

As we can see, the privacy leak depends on two terms. Term
$\frac{1}{\alpha_i^2}$ determines how close the client's message is to the true
mean. Term $\frac{\rho (\gamma_i + \tau)}{\gamma_i + \rho + \tau}$ describes
how well the messages of others estimate the true mean.

\paragraph{Participation Incentives} The utility function will have the form
\[
  u_i = \Par[\Big]{\frac{1}{\alpha_i^2} + \frac{\rho (\gamma_i +
  \tau)}{\gamma_i + \rho + \tau}}^{-1} - \frac{1}{\rho} - \frac{1}{\tau} -
  \frac{\lambda_i}{\gamma_i + \rho + \tau}.
\]
To compare this utility with the utility of training alone, we consider
limit $\forall k \: \beta_k \to 0$ (i.e., $\alpha_k \to \infty$). This limit
implies $\gamma_i \to 0$ and corresponds to utility $u^0_i =
-\frac{\lambda_i}{\rho + \tau}$. The clients agree to participate if the
following constraint holds
\begin{equation}
  \label{eq:b-mean-gen-system}
  \forall i \
  \frac{(\rho - \beta_i)^2}{(\varGamma + \tau) \rho^2} - \frac{\beta_i}{\rho^2}
  - \frac{\lambda_i}{\gamma_i + \rho + \tau} \ge \frac{1}{\tau} -
  \frac{\lambda_i}{\rho + \tau}.
\end{equation}

We will consider the behavior of this inequality in limit $\forall k \: \beta_k
\to 0$. This limit corresponds to the small perturbation of the status quo case
of non-collaboration or, alternatively, to the situation where participants
have a minimal privacy loss.

\begin{theorem}[Proof in \cref{sec:b-mean-first-order-proof}]
  \label{thm:b-mean-first-order}
  A small beneficial deviation for everyone exists if and only if
  \[
    \forall i \: \xi_i \ge 0 \text{ and } \sum_{i=1}^N \xi_i > 1,
  \]
  where $\xi_i \defeq \Par[\bigg]{1 + \frac{(\rho + \tau)^2}{\tau^2 \rho^2
  \Par[\big]{\frac{\lambda_i}{(\rho + \tau)^2} - \frac{1}{\tau^2}}}}^{-1}.$
\end{theorem}
Moreover, the proof of the theorem suggests that one possible beneficial
deviation will have the following form $\beta_i = \xi_i b$, where $b$ is
sufficiently small.

To interpret the results, notice that, in the first order
approximation, the following should hold
\[
  \xi_i \varGamma \ge \beta_i \iff \alpha_i^2 \ge \frac{1}{\xi_i \varGamma} -
  \frac{1}{\rho}, \: \varGamma \defeq \sum_{i=1}^N \beta_i.
\]
Since we consider the limit $\varGamma \to 0$, the first term will dominate the
second one, and the value $\nicefrac{1}{\xi_i}$ would again correspond to the
``minimal level'' of noise the entity should add compared to others. We want
these ``minimal levels'' of noise to be small to make participation optimal for
everyone.

\subsection{Limit Results for General Utility}
\label{sec:general_existence}
In all cases we analyzed, we observed that mutually beneficial protocols exist
when the number of players $N \to \infty$. We now prove this for a general
class of utility functions and accuracy and privacy losses. Consider any fixed
protocol $p^\text{sym}_{\alpha}$ with the same noise parameter $\alpha$ for
each player. We assume that $\err_i$ and $\leak_i$ as functions of $N$ and $\alpha$ are the same across all participants. We also make the following minimal assumptions.
\begin{assumption}
  \label{ass:general_assumptions}
  \begin{enumerate}
    \item $u_i$ is monotonically decreasing and continuous in both arguments
      and $u_i(0, 0) > u_i(\err, \leak), \forall (\err, \leak) \neq (0, 0)$.
    \item $\err_i(p^\text{sym}_{\alpha})$ decreases strictly and monotonically
      to $0$ as $N \to \infty$ and $\leak_i(p^\text{sym}_{\alpha})$ decreases
      strictly and monotonically to $0$ as $\alpha \to \infty$.
      \item $\err_i^0 \geq \err^0$ for all $i$ for some $\err^0 > 0$. Thus, no player can learn an arbitrary accurate model alone.
  \end{enumerate}
\end{assumption}
These natural assumptions describe the preferences of the clients towards
improved accuracy and privacy and the fact that accuracy gains improve with the
number of players. They also reflect the natural impact of noise on accuracy
and privacy. We have the following result.
\begin{theorem}[Proof in \cref{sec:existence_mean-proof}]
\label{thm:existence_mean}
  Assume that the utilities of all players belong to a finite set of possible
  utility functions $U$. Assume that every $u\in U$ and all functions $\err_i$
  and $\leak_i$ satisfy Assumption \ref{ass:general_assumptions}. With all
  other parameters fixed, there exist values $N_1 \in \mathbb{N}$ and $\alpha
  \in (0, \infty)$, such that whenever $N \geq N_1$ players are available with
  utilities from $U$, setting $\alpha_i = \alpha$ for all players $i\in [N]$
  ensures that the protocol $p^\text{sym}_{\alpha}$ is mutually beneficial.
\end{theorem}

\section{OPTIMAL PROTOCOLS FOR UTILITY}
\label{sec:util-max-sym}

In this section, we consider how the server could optimize the utilities of all
participants. To do this, we assume that the objective of the server is the
linear combination of the clients utilities
\[
  \max_{p_{\vec{\alpha}}} F(p_{\vec{\alpha}}) \defeq \sum_{i=1}^N
  \nu_i u_i(p_{\vec{\alpha}}) \text{ s.t. } u_i(p_{\vec{\alpha}}) \ge u^0_i.
\]
For simplicity, we consider the symmetric case $\forall i \:
\lambda_i = \lambda, \alpha_i = \alpha$, which will imply the same utilities
for all participants $\forall i \: u_i(p_{(\alpha, \dots, \alpha)^\tran})
\eqdef u^{\text{sym}}(\alpha)$. Setting $u^{0, \text{sym}} = \min_i u^0_i$ we arrive at the following objective
\[
  \max_{\alpha} F^{\text{sym}}(\alpha) \defeq u^{\text{sym}}(\alpha) \text{
  s.t. } u^{\text{sym}}(\alpha) \ge u^{0, \text{sym}}.
\]

\paragraph{DP Mean Estimation} The next Theorem describes the choice of $\beta$ (and therefore $\alpha$) parameters that is optimal for the utility maximization problem.
\begin{theorem}[Proof in \cref{sec:dp-mean-symmetric-general-proof}]
  \label{thm:dp-mean-symmetric-general}
  If $(N - 1) \lambda \le \kappa^2 \rho^2$, the collaboration becomes
  unprofitable, $\beta^* = 0$.

  If $(N - 1) \lambda > \kappa^2 \rho^2$, the collaboration becomes profitable,
  $
    %\beta^* &= \Par[\bigg]{1 - \frac{N \kappa \rho}{\sqrt{(N - 1) \lambda} + (N
    %- 1) \kappa \rho}} \rho,\\
    (\alpha^*)^2 = \frac{N \kappa}{\sqrt{(N - 1) \lambda} - \kappa \rho}
  $.

  The optimal level of noise has the following properties
  \begin{align*}
    u_i(\alpha^*) - u^0_i &= \frac{(\sqrt{(N - 1) \lambda} - \kappa \rho)^2}{N
    \rho}, \: \odv{(\alpha^*)^2}{\lambda} \le 0, \\
    \sign \odv{(\alpha^*)^2}{N} &= \sign((N - 2) \sqrt{\lambda} - 2
    \sqrt{N - 1} \kappa \rho).
  \end{align*}
\end{theorem}
As we can see, collaboration becomes profitable if the number of clients $N$ is large, the accuracy incentives $\lambda$ are strong, the privacy
concerns $\kappa$ are small, or the variance of local estimates $\rho$ is large. Naturally, the optimal noise level decrease when the accuracy incentives $\lambda$ are large or the privacy concerns $\kappa$ are small. At the same time, when the number of participants $N$ is
large, the optimal noise level increases, since sufficient accuracy gains are still present due to the large number of updates. We also discuss the dependence of $\alpha^*$ on other parameters the supplementary material.

\paragraph{Symmetric DP Stochastic Optimization} We have
\begin{theorem}[Proof in \cref{sec:dp-sgd-symmetric-proof}]
  \label{thm:dp-sgd-symmetric}
  If $\sqrt{(N - 1) \lambda} < \sqrt{\frac{4 N L \mu d}{N - 1}} \kappa \rho$,
  collaboration is unprofitable. If $\sqrt{(N - 1) \lambda} \ge \sqrt{\frac{4 N L \mu d}{N - 1}} \kappa \rho$,
  collaboration is profitable with $
    %\beta^* = \frac{\sqrt{\lambda} \rho}{\sqrt{\lambda} + \sqrt{2 \mu d N}
    %\kappa \rho}, \:
    (\alpha^*)^2 = \frac{\sqrt{L \mu N} \kappa}{\sqrt{\lambda d}}
    %&= \frac{32 \sqrt{\me \mu N \ln(1.25 T n^2) \ln(4 n^2) T} L
    %B}{\sqrt{\lambda d} n}.
  $.

  The optimal level of noise satisfies
  \begin{align*}
    \odv{(\alpha^*)^2}{\lambda} &\le 0, \: \odv{(\alpha^*)^2}{\sigma^2} \le 0,
    \: \odv{(\alpha^*)^2}{N} \ge 0,\\
    u_i(\alpha^*) &- u^0_i =\\
    &\frac{\sqrt{(N - 1) \lambda}}{L \mu \rho N} \Par[\bigg]{\sqrt{(N - 1)
    \lambda} - \sqrt{\frac{4 N L \mu d}{N - 1}} \kappa \rho}.
  \end{align*}
\end{theorem}

As we can see, the results are qualitatively similar to the results of
\cref{thm:dp-mean-symmetric-general}.

\paragraph{Symmetric Bayesian Mean Estimation} First, we give some necessary and sufficient conditions for the existence of
non-trivial solution and some general insights about the optimal level of
noise.

\begin{theorem}[Proof in \cref{sec:b-mean-symmetric-general-proof}]
  \label{thm:b-mean-symmetric-general}
  If $\lambda < \frac{(N \rho + \tau)^2}{(N - 1)^2 \rho^2}$, it is not
  profitable to collaborate, $\beta^* = 0$.

  If $\lambda > \frac{(N \rho^2 + 2 \rho \tau + \tau^2) (\rho + \tau)^2}{(N -
  1) \rho^2 \tau^2}$, it is always profitable to collaborative, $\beta^* > 0$.
  If also $\lambda < \frac{(N \rho + \tau)^2}{(N - 1) \rho^2}$, $\beta^*$ will
  be the only solution of equation $\odv{u_i}{\beta} = 0$ on the interval $[0,
  \rho]$ and will have the following properties
  \[
    \odv{\beta^*}{\lambda} \ge 0, \: \odv{\beta^*}{\rho} \ge 0.
  \]

  If $\lambda > \frac{(N \rho^2 + 2 \rho \tau + \tau^2) (\rho +
  \tau)^2}{(N - 1) \rho^2 \tau^2}$ and $\lambda \ge \frac{(N \rho + \tau)^2}{(N
  - 1) \rho^2}$, collaboration will be profitable and moreover nobody will add
  noise to their messages, $\beta^* = \rho$.
\end{theorem}

As we can see, when $\lambda$ is small, collaboration becomes unprofitable
for everybody. On the other hand, when $\lambda$ is big, people would want
to collaborate and moreover will not want to trade-off accuracy gain for
privacy. As expected, an increase in accuracy concerns incentivize people
to send less noise to the server. Interestingly, an increase in informativeness
of local estimate (increase in $\rho$) also makes their messages more
informative.

%\begin{remark}
%  In general, the property $\odv{\beta^*}{\rho} \ge 0$ does not imply that
%  $\odv{\alpha^*}{\rho} \le 0$ since $\pdv{\beta^*}{\rho} > 0$. Also, notice
%  that the restriction $\lambda > \frac{(N \rho^2 + 2 \rho \tau + \tau^2) (\rho
%  + \tau)^2}{(N - 1) \rho^2 \tau^2}$ is equivalent to the restriction
%  $\sum_{i=1}^N \xi_i > 1$.
%\end{remark}

Now, we will look at the limit behavior of the solution when $N \to \infty$,
$\rho \to \infty$, and $\rho \to 0$.

\begin{theorem}[Proof in \cref{sec:b-mean-limits-proof}]
  \label{thm:b-mean-limits}
  When $\lambda > \frac{\rho + \tau}{\tau}$, in the limit $N \to \infty$,
  it is profitable to collaborate, $\beta^* = \sqrt{\frac{\lambda - 1}{N}}
  \rho + \o\Par[\big]{\frac{1}{\sqrt{N}}}$,
  \[
    u_i(\beta^*) - u^0_i = \frac{\lambda}{\rho + \tau} - \frac{1}{\tau} +
    \o(1).
  \]
  
  When $\lambda < \frac{\rho + \tau}{\tau}$, in the limit $N \to \infty$,
  it is unprofitable to collaborate, $\beta^* = 0$. In the limit $\rho \to \infty$, it is unprofitable to collaborate, $\beta^* =
  0$. In the limit $\rho \to 0$, collaboration is also unprofitable, $\beta^* =
  0$.
\end{theorem}

Similarly to \cref{thm:dp-mean-symmetric-general}, when the number of
participants is big enough, people will want to join the collaborative learning
procedure and shrink their contributions with the number of participants.
The second property shows that when people have very good local estimates they
do not want to participate in collaborative learning: the potential gain in
accuracy is small, while the privacy concerns become very big.
The last property shows that, when potential gain from learning is very small,
$\rho \to 0$, privacy concerns start to dominate learning concerns.

\section{OPTIMAL PROTOCOLS FOR ACCURACY}
\label{sec:acc-max}

\begin{figure}[t]
  \centering 
  \input{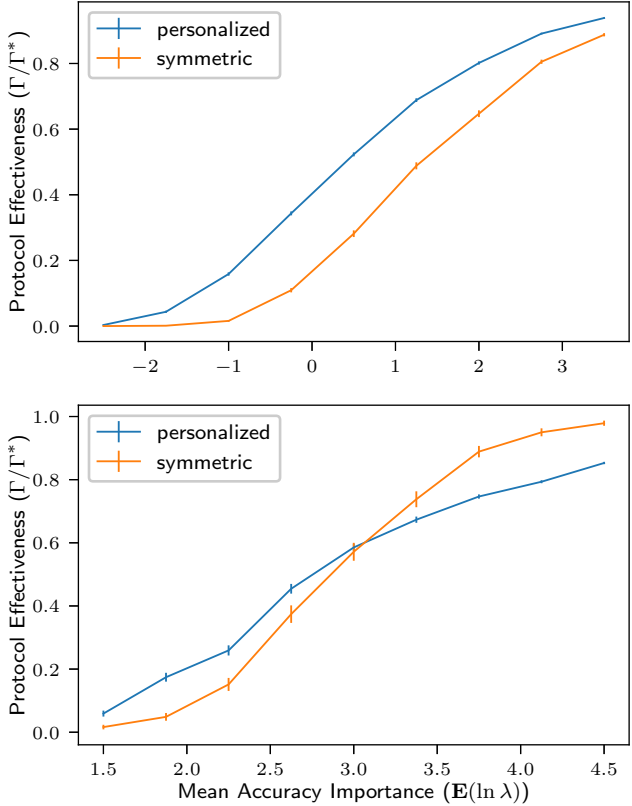}
  \caption{Personalized vs symmetric protocols. Error bars depict the standard
  deviation of the average effectiveness.}
  \label{fig:sym-vs-pers}
\end{figure}

Finally, we empirically find the optimal protocols for the server's
end-model accuracy. The server's accuracy in all cases depends only
on the value of $\varGamma = \sum \beta_i$ (essentially the inverse of the amount of noise) that the clients use. Thus, we will be
interested in maximizing $\varGamma$ given the participation constraints.

We focus on the mean estimation case and compare two families of
solutions. The first family are the symmetric solutions $\beta_i = b, b \in
[0, \rho]$, which are standard in practice. The second family are the solutions from our existence proofs, Theorems
\ref{thm:dp-mean-existence} and \ref{thm:b-mean-first-order}. In all cases, we seek a protocol that maximizes $\varGamma = \sum_{i=1}^N \beta_i$, by optimizing the parameter $\beta$ via a grid search. We focus on a setup with $N = 5$, $n = 100$, and $\sigma = 10$.
In each experiment, we sample the preferences for accuracy from the
log-normal distribution $\lambda_i \sim \exp(\omega)$, where $\omega \sim \N(m,
1)$. We repeat the experiment $1000$ for each value of $m$ and look at
the average ratio of the resulting $\varGamma$ over the optimal value of
$\varGamma^* = N \rho$.

\textbf{DP Mean Estimation} We compare the symmetric solutions to the
solutions found in \cref{thm:dp-mean-existence}, $\beta_i = \zeta_i b, b \in
[0, \nicefrac{\rho}{\max(\zeta_i)}]$. We set $B =
20$. We depict our results on \cref{fig:sym-vs-pers} (upper half). As we can
see, the personalized version of the protocol is always more beneficial than
the symmetric version, especially for small $\lambda$.

\textbf{Bayesian Mean Estimation} We compare the symmetric solutions to
the solutions inspired by \cref{thm:b-mean-first-order}, $\beta_i = \max(\xi_i,
0) b, b \in [0, \nicefrac{\rho}{\max(\xi_i)}]$. We set
$\tau = 1$. We depict our results on \cref{fig:sym-vs-pers} (lower half).  The personalized version of the protocol is more efficient than the
symmetric version of the protocol for small $\lambda$. However, for large $\lambda$ the trend is reversed, as the first-order approximation used for
\cref{thm:b-mean-first-order} is less accurate for large $\beta$.

\section{DISCUSSION}
This work studied the accuracy-privacy trade-off in federated learning as a
multi-objective utility maximization problem. First, we presented a general
framework that formalizes our multi-objective problem and our notion of
mutually beneficial protocols. Then, we derived necessary and sufficient
conditions for the existence of mutually beneficial protocols in the context of
mean estimation and strongly convex optimization. In the case of symmetric
privacy preferences, we found mutually beneficial protocols that are optimal
for total client utility. Finally, we studied optimal mutually beneficial
protocols for the end-model accuracy via a synthetic data experiment. Our
findings demonstrate that analyzing federated learning through the lens of
utility-based analysis can guide the development of better FL protocols.

\paragraph{Future work} We see further practical exploration of our framework
as a promising direction for future work.

A key practical challenge is to apply our analysis to other privacy notions
\citep[e.g., Renyi DP][]{m17r} and non-convex problems. In particular, for
large-scale applications, the known theoretical upper bounds for accuracy are
often loose or inapplicable, complicating the $\err$-function choice. A
practical solution might involve running small-scale tests to get an empirical
scaling law for error. We deem the exploration of such practical approaches to
our multi-objective problem a promising direction for future work.

Another practical challenge is to study the variations of the classic Fed-SGD
algorithm. In particular, it would be interesting to study weight communication
\citep{m17c, l20f, l21f, k21f} and co-training methods \citep{b20d, a23p},
which have been developed to decrease privacy risks in FL. Even in
such setups, privacy remains a concern \citep{g21t,d22d,z23s}, highlighting
the fundamental nature of the privacy-accuracy trade-off and the need for the
instruments for its analysis.

Finally, our results assume that the agents know their privacy preferences in
terms of $\lambda$ parameter (e.g., due to innate preferences or personality
traits). However, in cross-silo settings, clients’ privacy preferences
might instead come from profit maximization incentives (e.g., little privacy
protection may result in the loss of clients and, therefore, revenue). We see
the development of such micro-foundation models for privacy preferences as an
important direction for future work.

\section*{Acknowledgments}

This research has received funding by the Ministry of Education and Science of Bulgaria (support for
INSAIT, part of the Bulgarian National Roadmap for Research Infrastructure). Anna Mihalkova and Teodora Todorova conducted most of their work on this paper while on an internship at INSAIT, Sofia University. The authors would
like to thank Florian Dorner and Dimitar I. Dimitrov for their helpful
feedback and discussions in relation to this work. 

\bibliography{bibliography}
\bibliographystyle{plainnat}

\section*{Checklist}

\begin{enumerate}

  \item For all models and algorithms presented, check if you include:
    \begin{enumerate}
      \item A clear description of the mathematical setting, assumptions, algorithm, and/or model. [Yes]
      \item An analysis of the properties and complexity (time, space, sample size) of any algorithm. [Not Applicable]
      \item (Optional) Anonymized source code, with specification of all dependencies, including external libraries. [Not Applicable]
    \end{enumerate}

  \item For any theoretical claim, check if you include:
    \begin{enumerate}
      \item Statements of the full set of assumptions of all theoretical results. [Yes]
      \item Complete proofs of all theoretical results. [Yes]
      \item Clear explanations of any assumptions. [Yes]     
    \end{enumerate}

  \item For all figures and tables that present empirical results, check if you include:
    \begin{enumerate}
      \item The code, data, and instructions needed to reproduce the main
        experimental results (either in the supplemental material or as a URL).
        [Yes]
      \item All the training details (e.g., data splits, hyperparameters, how they were chosen). [Not Applicable]
      \item A clear definition of the specific measure or statistics and error bars (e.g., with respect to the random seed after running experiments multiple times). [Yes]
      \item A description of the computing infrastructure used. (e.g., type of GPUs, internal cluster, or cloud provider). [Not Applicable]
    \end{enumerate}

  \item If you are using existing assets (e.g., code, data, models) or curating/releasing new assets, check if you include:
    \begin{enumerate}
      \item Citations of the creator If your work uses existing assets. [Not Applicable]
      \item The license information of the assets, if applicable. [Not Applicable]
      \item New assets either in the supplemental material or as a URL, if applicable. [Not Applicable]
      \item Information about consent from data providers/curators. [Not Applicable]
      \item Discussion of sensible content if applicable, e.g., personally identifiable information or offensive content. [Not Applicable]
    \end{enumerate}

  \item If you used crowdsourcing or conducted research with human subjects, check if you include:
    \begin{enumerate}
      \item The full text of instructions given to participants and screenshots. [Not Applicable]
      \item Descriptions of potential participant risks, with links to Institutional Review Board (IRB) approvals if applicable. [Not Applicable]
      \item The estimated hourly wage paid to participants and the total amount spent on participant compensation. [Not Applicable]
    \end{enumerate}

\end{enumerate}

\newpage

\onecolumn

\appendix

\begin{center}
  {\LARGE Supplementary Material}
\end{center}

The supplementary material is structured as follows.
\begin{itemize}
\item \cref{sec:background} contains several existing results on differential privacy, as well as definitions from convex optimization that are used in our analysis.
  \item \cref{sec:proofs} contains the proofs of all results in the main text.
  \item \cref{sec:other_utilities} presents further analysis with another choice of utility function. 
\end{itemize}

\section{BACKGROUND}
\label{sec:background}

\paragraph{Differential privacy} We first present a classic result about the DP guarantees of the Gaussian mechanism.

\begin{theorem}[Theorem A.1, \citealt{d14a}]
  \label{thm:dp-mean-privacy}
  Let $\varepsilon \in (0, 1)$ be arbitrary. For $c^2 > 2
  \ln\Par*{\frac{1.25}{\delta}}$, the Gaussian Mechanism for function $f$
  with parameter $\sigma \ge \frac{c \Delta_2 f}{\epsilon}$, where $\Delta_2 f$
  is the $\ell_2$ sensitivity of $f$, is $(\varepsilon, \delta)$-differentially
  private.
\end{theorem}

In the case of mean estimation, $f$ is the average over all samples, $\Delta_2
f = \frac{B}{n}$, $\delta = \frac{1}{n^2}$, and $\sigma = \alpha_i$. In the
case of stochastic optimization, $f$ is the average over all gradients,
$\Delta_2 f = \frac{B}{n}$, and $\sigma = \alpha_i$.

Additionally, we state a result from \citealt{f22h}, which provides DP guarantees for chaining several local-DP algorithms. 

\begin{theorem}[Theorem 3.8, \citealt{f22h}]
  \label{thm:dp-sgd-suffle-privacy}
  For a domain $\mathcal{D}$, let $R^t\colon f \times \mathcal{D} \to
  S^t$ for $t \in [m]$ (where $S^t$ is the range space of $R^t$) be a
  sequence of algorithms such that $R^t(z_{1:t-1}, \cdot)$ is a $(\varepsilon_0,
  \delta_0)$-DP local randomizer for all values of auxiliary inputs $z_{1:t-1}
  \in S^1 \times \dots \times S^{t-1}$. Let $A\colon \mathcal{D}^m \to S^1
  \times \dots \times S^m$ be the algorithm that given a dataset $b_{1:m} \in
  B$, samples a uniformly random permutation $\pi$, then sequentially
  computes $z^t = R^t(z_{1:t-1}, b_{\pi(t)})$ for $t \in [m]$ and
  outputs $z_{1:m}$. Then for any $\delta \in [0, 1]$ such that $\varepsilon_0
  \le \ln\Par[\Big]{\frac{n}{16 \ln\Par*{\frac{2}{\delta}}}}$, $A$ is
  $(\varepsilon, \delta + (\me^\varepsilon + 1) (1 +
  \me^{-\frac{\varepsilon_0}{2}}) m \delta_0)$-DP, where $\varepsilon$ has the
  following property
  \[
    \varepsilon \le \ln\Par[\bigg]{1 + \frac{\me^{\varepsilon_0} -
    1}{\me^{\varepsilon_0} + 1} \Par[\bigg]{\frac{8 \sqrt{\me^{\varepsilon_0}
    \ln(\nicefrac{4}{\delta})}}{\sqrt{m}} + \frac{8
    \me^{\varepsilon_0}}{m}}}.
  \]
\end{theorem}

In our case, the dataset $B$ corresponds to the batches that the client will
use during training. Two batches will be neighboring if they differ in one
element and this elements are neighboors in the usual sense. At the same time,
$R^t$ corresponds to the computation of gradient using the client's batch.
(Here, the contribution of other clients are treated as an additional
randomization in the algorithm $R^t$.)

\paragraph{Assumptions for the SGD objective} We assume that $f$ is $\mu$-strongly convex on $W$:
$$\forall \w_1, \w_2 \in W \quad f(\w_1) \geq f(\w_2) + \nabla f(\w_2)^{\tran}(\w_1 - \w_2) + \frac{\mu}{2}\|\w_1 - \w_2\|^2$$
as well as $L$-smooth on $\mathbb{R}^d$:
$$\forall \w_1, \w_2 \in \mathbb{R}^d \quad f(\w_1) \leq f(\w_2) + \nabla f(\w_2)^{\tran}(\w_1 - \w_2) + \frac{L}{2}\|\w_1 - \w_2\|^2.$$

\section{PROOFS}
\label{sec:proofs}

\subsection{Proof of \texorpdfstring{\cref{thm:dp-mean-predictor}}{Theorem
\ref{thm:dp-mean-predictor}}}
\label{sec:dp-mean-predictor-proof}

We want to solve the following optimization problem
\[
  \min_{\{a_i\}_{i=1}^N} \E\Par[\bigg]{\Par[\bigg]{\sum_{k \neq i} a_k m_k +
  a_i \bar{x}_i - \mu}^2} \text{ s.t. } \sum_{i=1}^N a_i = 1.
\]
Since $\{m_k\}_{k \neq i} \cup \{\bar{x}_i\}$ are independent, we get
\[
  \E\Par[\bigg]{\Par[\bigg]{\sum_{k \neq i} a_k m_k + a_i \bar{x}_i - \mu}^2} =
  \sum_{k \neq i} a_k^2 \Par*{\alpha_k^2 + \frac{1}{\rho}} +
  \frac{a_i^2}{\rho},
\]
where $\rho \defeq \frac{n}{\sigma^2}$. Therefore, we get the following first
optimality order conditions
\begin{align*}
  2 a_k \Par*{\alpha_k^2 + \frac{1}{\rho}} &= \nu \: \forall k \neq i, \\
  \frac{2 a_i}{\rho} &= \nu,
\end{align*}
where $\nu$ is a Lagrange multiplier. Solving these equations, we get
\begin{align*}
  a_k &= \frac{\beta_k}{\gamma_i + \rho} \: \forall k \neq i, \\
  a_i &= \frac{\rho}{\gamma_i + \rho},
\end{align*}
where $\beta_k \defeq \frac{1}{\alpha_k^2 + \frac{1}{\rho}}$ and $\gamma_i
\defeq \sum_{k \neq i} \beta_k$. These weights give the desired formula for the
optimal estimator.

\subsection{Proof of \texorpdfstring{\cref{thm:dp-mean-existence}}{Theorem
\ref{thm:dp-mean-existence}}}
\label{sec:dp-mean-existence-proof}

Denote $\varGamma \defeq \sum_{i=1}^N \beta_i = \gamma_i + \beta_i$. Then we
get the following expression for the system \cref{eq:dp-mean-gen-system}
\[
  \forall i \: \frac{\lambda_i (\varGamma - \beta_i)}{\rho (\varGamma - \beta_i
  + \rho)} \ge \frac{\kappa^2 \rho \beta_i}{\rho - \beta_i} \iff
  \zeta_i \varGamma \rho - (\varGamma + \rho) \beta_i + \beta_i^2 \ge 0,
  \text{ where } \zeta_i = \frac{\lambda_i}{\lambda_i + \kappa^2 \rho^2}.
\]
Thus,
\[
  \Par*{\sum_{i=1}^N \zeta_i} \varGamma \rho - \varGamma (\varGamma + \rho) +
  \sum_{i=1}^N \beta_i^2 \ge 0 \implies \Par*{\sum_{i=1}^N \zeta_i} \varGamma
  \rho - \varGamma (\varGamma + \rho) + \varGamma^2 > 0,
\]
which gives you the following necessary condition
\[
  \sum_{i=1}^N \zeta_i > 1.
\]

This condition is also sufficient. To prove this property, we consider the
following parametric family $\beta_i = \zeta_i b$. It gives the following
version of \cref{eq:dp-mean-gen-system}
\[
  \zeta_i \Par*{\sum_{k=1}^N \zeta_k} b \rho - \zeta_i b \Par*{\rho +
  \Par*{\sum_{k=1}^N \zeta_k} b} + \zeta_i^2 b^2 \ge 0
  \iff \zeta_i \Par*{\Par*{\sum_{k=1}^N \zeta_k} - 1} b \rho - \zeta_i
  \Par*{\sum_{k \neq i} \zeta_k} b^2 \ge 0
\]
As we can see, this equation holds for $b \in \Br*{0, \Par*{1 -
\frac{1}{\sum_{i=1}^N \zeta_i}} \rho}$.

\subsection{Proof of
\texorpdfstring{\cref{thm:dp-mean-symmetric-general}}{Theorem
\ref{thm:dp-mean-symmetric-general}}}
\label{sec:dp-mean-symmetric-general-proof}

According to \cref{thm:dp-mean-existence}, we have solutions if and only if
\[
  (N - 1) \lambda > \kappa^2 \rho^2,
\]
which proves the first statement in the theorem.

In symmetric case, utility gain has the following form
\[
  h_\lambda(\beta) \defeq -\frac{\kappa^2 \rho \beta}{\rho - \beta} + \frac{(N
  - 1) \lambda \beta}{\rho ((N - 1) \beta + \rho)} = \frac{\lambda}{\rho} +
  \kappa^2 \rho - \frac{\lambda}{(N - 1) \beta + \rho} - \frac{\kappa^2
  \rho^2}{\rho - \beta}.
\]
To find optimal amount of noise $\beta^*$, we should solve the first order
condition for the extremum of $h_\lambda$
\[
  0 = h'_\lambda(\beta^*) = \frac{\lambda (N - 1)}{((N - 1) \beta^* + \rho)^2}
  - \frac{\kappa^2 \rho^2}{(\rho - \beta^*)^2},
\]
which gives
\[
  \beta^* = \frac{(\sqrt{(N - 1) \lambda} - \kappa \rho) \rho}{\sqrt{(N - 1)
  \lambda} + (N - 1) \kappa \rho}.
\]
Using definition of $\beta^*$, we get the following expression for
$(\alpha^*)^2$
\[
  (\alpha^*)^2 = \frac{1}{\beta^*} - \frac{1}{\rho} = \frac{N}{\frac{\sqrt{(N -
  1) \lambda}}{\kappa} - \rho} = \frac{N}{\frac{\sqrt{(N - 1) \lambda}
  n}{\sqrt{2 \ln(1.25 n^2)} B} - \frac{n}{\sigma^2}}.
\]

It is easy to see that $\alpha^*$ is decreasing in $\lambda$ and $\sigma$. By
direct calculations,
\begin{multline*}
  \odv{(\alpha^*)^2}{N} = \frac{1}{\frac{\sqrt{(N - 1) \lambda}}{\kappa} -
  \rho} - \frac{N \sqrt{\lambda}}{2 \sqrt{N - 1} \kappa \Par*{\frac{\sqrt{(N -
  1) \lambda}}{\kappa} - \rho}^2}
  = \frac{(N - 2) \sqrt{\lambda} - 2 \sqrt{N - 1} \kappa \rho}{2 \sqrt{N - 1}
  \kappa \Par*{\frac{\sqrt{(N - 1) \lambda}}{\kappa} - \rho}^2} \\
  \implies
  \sign\Par*{\odv{(\alpha^*)^2}{N}} = \sign\Par*{(N - 2) \sqrt{\lambda} - 2
  \sqrt{N - 1} \kappa \rho}.
\end{multline*}
To show find the sign of $\odv{(\alpha^*)^2}{n}$, we will analyze the
derivative of inverse
\begin{multline*}
  N \odv{\frac{1}{(\alpha^*)^2}}{n} = \frac{\sqrt{(N - 1) \lambda}}{\sqrt{2
  \ln(1.25 n^2)} B} - \frac{1}{\sigma^2} - \frac{\sqrt{(N - 1)
  \lambda}}{\ln(1.25 n^2) \sqrt{2 \ln(1.25 n^2)} B}
  = \frac{1}{n} \Par*{\frac{\sqrt{(N - 1)
  \lambda}}{\kappa} - \rho - \frac{\sqrt{(N - 1) \lambda}}{\ln(1.25 n^2)
  \kappa}} \\
  \implies \sign\Par*{\odv{(\alpha^*)^2}{n}} = \sign\Par*{\kappa \rho +
  \frac{\sqrt{(N - 1) \lambda}}{\ln(1.25 n^2)} - \sqrt{(N - 1) \lambda}}.
\end{multline*}

Finally, we will calculate the utility gain form using the optimal amount of
noise. First, notice that
\begin{align*}
  \rho - \beta^* &= \frac{N \kappa \rho^2}{\sqrt{(N - 1) \lambda} + \kappa
  \rho}, \\
  (N - 1) \beta^* + \rho &= \frac{N \sqrt{(N - 1) \lambda} \rho}{\sqrt{(N - 1)
  \lambda} + \kappa \rho}.
\end{align*}
Thus,
\begin{align*}
  \frac{\beta^*}{\rho - \beta^*} &= \frac{\sqrt{(N - 1) \lambda} - \kappa
  \rho}{N \kappa \rho}, \\
  \frac{\beta^*}{(N - 1) \beta^* + \rho} &= \frac{\sqrt{(N - 1) \lambda} -
  \kappa \rho}{N \sqrt{(N - 1) \lambda}}.
\end{align*}
And we finally get
\[
  h_\lambda(\beta^*) = \frac{(\sqrt{(N - 1) \lambda} - \kappa \rho)^2}{N \rho}.
\]

\subsection{Proof of \texorpdfstring{\cref{thm:dp-sgd-final-acc}}{Theorem
\ref{thm:dp-sgd-final-acc}}}
\label{sec:dp-sgd-final-acc-proof}

First, we will find the weights $a_i$ that will give the best unbiased estimate
of the gradient.

\begin{lemma}
  Optimal weights $a_i$ have the following properties
  \[
    a_i = \frac{\beta_i}{\varGamma}, \: \Var{\vec{g}^{t+1}} =
    \frac{1}{\varGamma},
  \]
  where $\rho = \frac{b}{\sigma^2}$, $\beta_i \defeq \frac{1}{\frac{1}{\rho} +
  d \alpha_i^2}$, and $\varGamma = \sum_{i=1}^N \beta_i$.
\end{lemma}

\begin{proof}
  Similarly to the differentially private estimation case, we want to solve the
  following optimization problem
  \[
    \min_{\{a_i\}_{i=1}^N} \E\Par[\bigg]{\norm[\bigg]{\sum_{k=1}^N a_k
    \vec{m}^t_k - \nabla f(\w^t)}^2} \text{ s.t. } \sum_{i=1}^N a_i = 1.
  \]
  Since messages are independent, we get
  \[
    \E\Par[\bigg]{\norm[\bigg]{\sum_{k=1}^N a_k \vec{m}^t_k - \nabla
    f(\w^t)}^2} = \sum_{k=1}^N a_k^2 \Par*{d \alpha_k^2 + \frac{1}{\rho}},
  \]
  where $\rho \defeq \frac{b}{\sigma^2}$. Therefore, we get the following first
  optimality order conditions
  \[
    2 a_k \Par*{d \alpha_k^2 + \frac{1}{\rho}} = \nu \: \forall k,
  \]
  where $\nu$ is a Lagrange multiplier. Solving these equations, we get
  \[
    a_k = \frac{\beta_k}{\varGamma} \: \forall k,
  \]
  where $\beta_k \defeq \frac{1}{d \alpha_k^2 + \frac{1}{\rho}}$ and $\varGamma
  \defeq \sum_{k=1}^N \beta_k$. These weights give the desired formula for
  the optimal estimator.
\end{proof}

Using this estimate for the variance of gradients, we get the following upper
bound on the objective.
\begin{lemma}
  \label{lem:dp-sgd-upper-bound}
  If $\eta^{t+1} \le \frac{2}{L}$, the iterates will satisfy the following
  inequality
  \[
    \Delta \w^{t+1} \le (1 - 2 \eta^{t+1} \mu + L \mu (\eta^{t+1})^2)
    \Delta \w^t + \frac{(\eta^{t+1})^2}{\varGamma},
  \]
  where $\Delta \w^\tau \defeq \E(\norm{\w^\tau - \w^*}^2)$ and $\w^* \defeq
  \argmin_{\w \in W} f(\w)$.
\end{lemma}

Before proving it, we derive the following standard lemma.

\begin{lemma}
  \label{lem:dp-sgd-grad-norm}
  Let $f^* = \min_{\mathbb{R}^d} f(\w)$. If $f$ is $L$-smooth,
  \[
    \norm{\nabla f(\w)}^2 \le 2 L (f(\w) - f^*).
  \]
\end{lemma}

\begin{proof}
  Denote $\vec{g} = \nabla f(\w)$. The smoothness gives
  \[
    f\Par*{\w - \frac{1}{L} \vec{g}} \le f(\w) - \frac{1}{2 L}
    \norm{\vec{g}}^2.
  \]
  Therefore,
  \[
    f^* - f(\w) \le f\Par*{\w - \frac{1}{L} \vec{g}} - f(\w) \le - \frac{1}{2
    L} \norm{\vec{g}}^2.
  \]
\end{proof}

Now, we will proof \cref{lem:dp-sgd-upper-bound}.
\begin{proof}
  We have
  \begin{multline*}
    \Delta \w^{t+1} = \E(\norm{\varPi_W(\w^t - \eta^{t+1} \vec{g}^{t+1}) -
    \w^*}^2) \le \E(\norm{\w^t - \eta^{t+1} \vec{g}^{t+1} - \w^*}^2) \\
    = \Delta \w^t - 2 \eta^{t+1} \E((\vec{g}^{t+1})^\tran (\w^t - \w^*)) +
    (\eta^{t+1})^2 \E(\norm{\vec{g}^{t+1}}^2) \\
    = \Delta \w^t - 2 \eta^{t+1} \E(\nabla f(\w^t)^\tran (\w^t - \w^*)) +
    (\eta^{t+1})^2 \Par*{\E(\norm{\nabla f(\w^t)}^2) + \frac{1}{\varGamma}} \\
    \le \Delta \w^t - 2 \eta^{t+1} \Par*{\E(f(\w^t) - f^*) + \frac{\mu}{2}
    \Delta \w^t} + (\eta^{t+1})^2 \Par*{2 L \E (f(\w^t) - f^*) +
    \frac{1}{\varGamma}}\\
    \le (1 - 2 \eta^{t+1} \mu + \mu L (\eta^{t+1})^2) \Delta \w^t +
    \frac{(\eta^{t+1})^2}{\varGamma},
  \end{multline*}
  where we used strong convexity and \cref{lem:dp-sgd-grad-norm} in the
  penultimate inequality and strong convexity and optimality in the last one.
\end{proof}

To study the sequence $\Delta \w^t$, we consider the following sequence of
upper bounds
\begin{align*}
  y^0 &\defeq \frac{L \mu N n D^2}{\sigma^2},\\
  y^{t+1} &\defeq (1 - 2 \eta^{t+1} \mu + L \mu (\eta^{t+1})^2) y^t +
  L \mu (\eta^{t+1})^2.
\end{align*}

\begin{lemma}
  \label{thm:dp-sgd-bounds-seq}
  For any choice of the step-sizes, we have $y^t \ge L \mu \varGamma \Delta
  \w^t$. $\eta^{t+1}_*$ that minimizes the sequence $y^t$ has the following
  properties
  \[
    \eta^{t+1}_* = \frac{y^t}{L (y^t + 1)} < \frac{2}{L}, \:
    (1 - 2 \eta^{t+1}_* \mu + L \mu (\eta^{t+1}_*)^2) y^t + L \mu
    (\eta^{t+1}_*)^2 = \Par*{1 - \frac{\chi y^t}{y^t + 1}} y^t,
  \]
  where $\chi = \frac{\mu}{L}$.
\end{lemma}

\begin{proof}
  First order conditions for the problem $\min_{\eta^{t+1}} y^{t+1}$ gives
  \[
    -\mu y^t + L \mu (1 + y^t) \eta^{t+1} = 0,
  \]
  which gives optimal $\eta^{t+1}_*$. Substituting it into a formula for
  $y^{t+1}$ gives the desired result.
\end{proof}

Now, we are ready to proof \cref{thm:dp-sgd-final-acc}.
\begin{proof}
  Notice that
  \[
    \Delta \w^0 \le D^2 \implies L \mu \varGamma \Delta \w^0 \le L \mu
    \varGamma D^2 \le \frac{L \mu n N D^2}{\sigma^2} = y^0_*.
  \]
  By induction,
  \begin{multline*}
    L \mu \varGamma \Delta \w^{t+1} \le (1 - 2 \eta^{t+1}_* \mu + \mu L
    (\eta^{t+1}_*)^2) L \mu \varGamma \Delta \w^t + L \mu (\eta^{t+1}_*)^2 \\
    \le (1 - 2 \eta^{t+1} \mu + \mu L (\eta^{t+1})^2) y^t_* + L \mu
    (\eta^{t+1}_*)^2 = y^{t+1}_*.
  \end{multline*}
  Which proves that $L \mu \varGamma \Delta \w^{t+1} \le y^{t+1}_*$.

  We know that
  \[
    y^{t+1}_* \le \Par*{1 - \frac{\chi y^t_*}{y^t_* + 1}} y^t_* = h(y^t_*).
  \]
  Notice that $h(y)$ is decreasing in $y$.

  Also notice that if $y^t_* \ge 1$, $y^{t+1}_* \le \Par*{1 - \frac{\chi}{2}}
  y^t_*$. Then for $t < T \defeq
  \max\Par[\Big]{\Ceil[\Big]{-\frac{\ln(y^0_*)}{\ln\Par*{1 - \frac{\chi}{2}}}},
  1}$, we have
  \[
    y^t_* \le \Par*{1 - \frac{\chi}{2}}^t y^0_*.
  \]
  For $t \ge T$, we have $y^t_* \le 1$.

  Consider the sequence $z^t \defeq \frac{1}{1 + \frac{\chi}{2 - \chi} t}$.
  We want to show that $z^t \ge y^{t+T}_*$. We already know that $z^0 = 1 \ge
  y^T_*$. Now notice that
  \begin{multline*}
    h(z^t) \le z^{t+1} \iff \Par*{1 - \frac{\chi}{2 + \frac{\chi}{2 - \chi}t}}
    \frac{1}{1 + \frac{\chi}{2 - \chi} t} \le \frac{1}{1 + \frac{\chi}{2 -
    \chi} (t + 1)} \iff \\
    \Par*{2 + \frac{\chi}{2 - \chi} t - \chi} \Par*{1 + \frac{\chi}{2
    - \chi} (t + 1)} \le \Par*{2 + \frac{\chi}{2 - \chi} t} \Par*{1 +
    \frac{\chi}{2 - \chi} t} \iff \frac{2 \chi^2 (1 - \chi)}{(2 - \chi)^2} t
    \ge 0,
  \end{multline*}
  which shows that $z^{t+1} \ge h(z^t)$. By induction, we have
  \[
    y^{t+1+T}_* \le h(y^{t+T}_*) \le h(z^t) \le z^{t+1}.
  \]

  Thus,
  \[
    \Delta w^t \le \frac{y^t_*}{L \mu \varGamma} \le
    \begin{cases}
      \frac{1}{\Par*{1 + \frac{\chi}{2 - \chi} (t - T)} L \mu \varGamma}, &
      t \ge T,\\
      \Par*{1 - \frac{\chi}{2}}^m \frac{y^0_*}{L \mu \varGamma}, & t < T.
    \end{cases}
  \]
\end{proof}

\subsection{Proof of \texorpdfstring{\cref{thm:dp-sgd-privacy}}{Theorem
\ref{thm:dp-sgd-privacy}}}
\label{sec:dp-sgd-privacy-proof}

First, notice that
\[
  \frac{\me^x - 1}{\me^x + 1} \le x.
\]
To prove it, we consider the following function
\[
  h(x) = x (\me^x + 1) - \me^x + 1.
\]
Notice that $h(0) = 0$ and 
\[
  h'(x) = \me^x + 1 + x \me^x - \me^x = 1 + x \me^x > 0.
\]
Therefore $h(x) \ge 0$, which proves the desired property.

Also notice that
\[
  \frac{\me^{\varepsilon_{0, i}}}{m} \le \frac{\me}{2} \le \ln(4).
\]

Now we would use \cref{thm:dp-sgd-suffle-privacy} to describe the global
privacy budget, $\varepsilon_i$, and \cref{thm:dp-mean-privacy} describe the
privacy budget of each iteration, $\varepsilon_{0, i}$.

Therefore,
\begin{multline*}
  \varepsilon_i \le \ln\Par[\bigg]{1 + \frac{\me^{\varepsilon_{0, i}} -
  1}{\me^{\varepsilon_{0, i}} + 1} \Par[\bigg]{\frac{8
  \sqrt{\me^{\varepsilon_{0, i}} \ln(\frac{4}{\delta})}}{\sqrt{m}} + \frac{8
  \me^{\varepsilon_{0, i}}}{m}}} \le \frac{\me^{\varepsilon_{0, i}} -
  1}{\me^{\varepsilon_{0, i}} + 1} \Par[\bigg]{\frac{8
  \sqrt{\me^{\varepsilon_{0, i}} \ln(\frac{4}{\delta})}}{\sqrt{m}} + \frac{8
  \me^{\varepsilon_{0, i}}}{m}} \\
  \le \varepsilon_{0, i} \frac{16 \sqrt{\me \ln(\frac{4}{\delta})}}{\sqrt{m}}
  \le \frac{\sqrt{2 \ln(1.25 m n^2)} B}{b \alpha_i} \frac{16 \sqrt{\me \ln(4
  n^2)}}{\sqrt{m}}.
\end{multline*}

\subsection{Proof of \texorpdfstring{\cref{thm:dp-sgd-existence}}{Theorem
\ref{thm:dp-sgd-existence}}}
\label{sec:dp-sgd-existence-proof}

Notice that \cref{eq:dp-sgd-gen-system} is equivalent to the system
\[
  \forall i \: \frac{\rho \psi_i \Par*{1 - \frac{\rho}{\varGamma}}}{1 +
  \psi_i \Par*{1 - \frac{\rho}{\varGamma}}} \ge \beta_i.
\]
Denote $x \defeq \frac{\varGamma}{\rho} - 1$. We get the system
\[
  \forall i \: \frac{\rho \psi_i x}{(\psi_i + 1) x + 1} \ge \beta_i.
\]
Summing all inequalities, we get the desired necessary condition
\[
  \sum_{i=1}^N \frac{\psi_i x}{(\psi_i + 1) x + 1} \ge \frac{\varGamma}{\rho}
  = x + 1.
\]

To prove that this condition is also sufficient, we consider
\[
  \beta_i = \frac{\rho \psi_i x^*}{(\psi_i + 1) x^* + 1} \le \rho,
\]
where $x^*$ solves the necessary condition. Then we get
\[
  \varGamma = \sum_{i=1}^N \beta_i = \sum_{i=1}^N \frac{\psi_i \rho
  x^*}{(\psi_i + 1) x^* + 1} \ge \rho (x^* + 1).
\]
Therefore, $\frac{\rho}{\varGamma} \le \frac{1}{x^* + 1}$ and
\[
  \frac{\rho \psi_i \Par*{1 - \frac{\rho}{\varGamma}}}{1 + \psi_i \Par*{1 -
  \frac{\rho}{\varGamma}}} \ge \frac{\rho \psi_i x^*}{x^* + 1 + \psi_i x^*}
  = \beta_i,
\]
which shows that the condition is also sufficient.

\subsection{Proof of
\texorpdfstring{\cref{cor:dp-sgd-existence-simple}}{Corollary
\ref{cor:dp-sgd-existence-simple}}}
\label{sec:dp-sgd-existence-simple-proof}

To prove the sufficient condition, we will require that $x = 1$ is a solution.
Then we should have
\[
  \sum_{i=1}^N \frac{\psi_i}{\psi_i + 2} \ge 2.
\]

To prove the necessary condition, notice that
\[
  \sum_{i=1}^N \frac{\psi_i x}{(\psi_i + 1) x + 1} \le \sum_{i=1}^N
  \frac{\psi_i x}{2 \sqrt{(\psi_i + 1) x}}.
\]
Therefore, it is necessary for the following system to have solutions
\[
  \sum_{i=1}^N \frac{\psi_i}{\sqrt{\psi_i + 1}} \ge 2 \Par*{\sqrt{x} +
  \frac{1}{\sqrt{x}}}.
\]
However, the system above have solutions if and only if
\[
  \sum_{i=1}^N \frac{\psi_i}{\sqrt{\psi_i + 1}} \ge 4.
\]

\subsection{Proof of \texorpdfstring{\cref{thm:dp-sgd-symmetric}}{Theorem
\ref{thm:dp-sgd-symmetric}}}
\label{sec:dp-sgd-symmetric-proof}

According to theorem \cref{thm:dp-sgd-existence}, \cref{eq:dp-sgd-gen-system}
has a solution if and only if inequality $\sum_{i=1}^N \frac{\psi x}{(\psi + 1)
x + 1} \ge x + 1$, where $\psi = \frac{\lambda}{L \mu d \kappa^2 \rho^2}$, has
a solution. In symmetric case, this inequality is equivalent to
\[
  N \psi x \ge (x + 1) ((\psi + 1) x + 1) \iff 0 \ge (\psi + 1) x^2 - ((N -
  1) \psi - 2) x + 1.
\]
The determinant of the left-hand part is equal to
\[
  ((N - 1) \psi - 2)^2 - 4 (\psi + 1) = ((N - 1)^2 \psi - 4 N) \psi.
\]
Therefore, we should have 
\[
  \sqrt{(N - 1) \psi} \ge \sqrt{\frac{4 N}{N - 1}} \iff \sqrt{(N - 1) \lambda}
  \ge \sqrt{\frac{4 N L \mu d}{N - 1}} \kappa \rho
\]
for the system to have a solution.

To find the optimal level of noise, we notice that utility gain is equal to the
following function in the symmetric case
\[
  h_\lambda(\beta) \defeq u_i - u^0_i = \kappa^2 \rho d \Par*{\psi \Par*{1 -
  \frac{\rho}{N \beta}} - \frac{\beta}{\rho - \beta}}.
\]
First-order conditions for the optimal level of noise give
\begin{align*}
  0 &= h'_\lambda(\beta^*) = \kappa^2 \rho d  \Par*{\frac{\psi \rho}{N
  (\beta^*)^2} - \frac{\rho}{(\rho - \beta^*)^2}} \implies \beta^* =
  \frac{\sqrt{\psi} \rho}{\sqrt{\psi} + \sqrt{N}} \implies \\
  (\alpha^*)^2 &= \frac{\sqrt{N}}{\sqrt{\psi} \rho d} = \frac{\sqrt{N L \mu}
  \kappa}{\sqrt{\lambda d}} = \frac{16 \sqrt{2 e N L \mu \ln(1.25 T n^2) \ln(4
  n^2) T} B}{\sqrt{\lambda d} n},
\end{align*}
where
\[
  T = \max\Par*{-\frac{\ln(y^0_*)}{\ln\Par*{1 - \frac{\chi}{2}}}, 1}, \:
  y^0_* = \frac{L \mu n N D^2}{\sigma^2}.
\]
It is easy to see that
\[
  \odv{(\alpha^*)^2}{\lambda} \le 0.
\]
Moreover, since $T$ is increasing in $N$ and decreasing in $\sigma^2$, we get
\[
  \odv{(\alpha^*)^2}{N} \ge 0, \: \odv{(\alpha^*)^2}{\sigma^2} \le 0.
\]
Also, we have
\[
  \odv{(\alpha^*)^2}{n} = -(\alpha^*)^2 \Par*{\frac{1}{n} - \frac{1}{2 \ln(1.25
  T n^2)} \Par*{\frac{2}{n} + \odv{T}{n}} - \frac{1}{\ln(4 n^2) n} -
  \frac{1}{2 T} \odv{T}{n}}.
\]
Since
\[
  \odv{T}{n} = -\frac{\Br*{\ln(y^0_*) \ge -\ln\Par*{1 - \frac{\chi}{2}}}}{n
  \ln\Par*{1 - \frac{\chi}{2}}},
\]
we have
\begin{multline*}
  \odv{(\alpha^*)^2}{n} = -\frac{(\alpha^*)^2}{n} \Par*{1 - \frac{1}{\ln(1.25 T
  n^2)} + \frac{\Br*{\ln(y^0_*) \ge -\ln\Par*{1 - \frac{\chi}{2}}}}{2 \ln(1.25
  n^2 T) \ln\Par*{1 - \frac{\chi}{2}}} - \frac{1}{\ln(4 n^2)} +
  \frac{\Br*{\ln(y^0_*) \ge -\ln\Par*{1 - \frac{\chi}{2}}}}{2 T \ln\Par*{1 -
  \frac{\chi}{2}}}} \\
  = -\frac{(\alpha^*)^2}{n} \Par*{1 - \O\Par*{\frac{1}{\ln(n)}}} \le 0 \text{
    for big enough $n$}.
\end{multline*}

Finally, we have
\[
  h_\lambda(\beta^*)  
  = \kappa^2 \rho d \Par*{\psi \frac{N - 1}{N} - \frac{2
  \sqrt{\psi}}{\sqrt{N}}} = \frac{\sqrt{(N - 1) \lambda}}{L \mu \rho N}
  \Par*{\sqrt{(N - 1) \lambda} - \sqrt{\frac{4 N L \mu d}{N - 1}} \kappa \rho}.
\]

\subsection{Proof of \texorpdfstring{\cref{thm:b-mean-privacy}}{Theorem
\ref{thm:b-mean-privacy}}}
\label{sec:b-mean-privacy-proof}

First, we want to show that it is optimal to use the same predictor for
different local points.

Let $\vec{x}_i \defeq (x^1_i, \dots, x^n_i)^\tran$ and $\hat{\vec{x}}_i \defeq
(\hat{x}^1_i, \dots, \hat{x}^n_i)^\tran$. Notice that
\[
  \vec{x}_i - \vec{\iota} \bar{x}_i = \mathbf{I} \vec{x}_i - \frac{\vec{\iota}
  \vec{\iota}^\tran}{n} \vec{x}_i = \Par*{\mathbf{I} - \frac{\vec{\iota}
  \vec{\iota}^\tran}{n}} \vec{x}_i.
\]
Therefore, $\vec{x}_i$ can be decomposed into the following sum
\[
  \vec{x}_i = \vec{\iota} \bar{x}_i + \sum_{j=1}^{n-1} \vec{v}^j \delta^j_i,
\]
where $\{\vec{v}^j\}$ are orthonormal eigenvectors of the matrix $\mathbf{I} -
\frac{\vec{\iota} \vec{\iota}^\tran}{n}$ that correspond to unit eigenvalue,
and $\delta^j_i$ are independent normal variables $\N(0, \sigma^2)$.

Let $\hat{\delta}^j_i = (\vec{v}^j)^\tran \hat{\vec{x}}_i$ and $\hat{\bar{x}}_i
= \frac{1}{n} \sum_{j=1}^n \hat{x}^j_i$. We have the following decomposition
\[
  \sum_{j=1}^n (\hat{x}^j_i - x^j_i)^2 = (\hat{\vec{x}}_i - \vec{x}_i)^\tran 
  (\hat{\vec{x}}_i - \vec{x}_i) = \sum_{j=1}^{n-1} (\hat{\delta}^j_i -
  \delta^j_i)^2 + n (\hat{\bar{x}}_i - \bar{x}_i)^2.
\]
Notice that $\hat{\vec{x}}_i$ can not depend on $\{\delta^j_i\}_{j=1}^{n-1}$.
Thus,
\[
  \E((\hat{\delta}^j_i - \delta^j_i)^2) = \E((\hat{\delta}^j_i)^2 +
  (\delta^j_i)^2) = \E((\hat{\delta}^j_i)^2) + \sigma^2.
\]
Therefore, the optimal predictor should have $\hat{\delta}^j_i = 0$. It allow
us to simplify privacy loss
\[
  \sum_{j=1}^n (\hat{x}^j_i - x^j_i)^2 = (n - 1) \sigma^2 + n (\hat{\bar{x}}_i
  - \bar{x}_i)^2.
\]
Corresponding best predictor should have a form $\hat{\vec{x}}_i = \vec{\iota}
\hat{\bar{x}}_i$.

Now, we want to derive Bayes estimator for $\bar{x}_i$ that only uses variables
$\{m_1, \dots, m_N\}$. By integrating over irrelevant noise, we get
\[
  \forall k \neq i \: m_k \sim \N\Par*{\mu, \frac{\sigma^2}{n} + \alpha_k^2},
  m_i \sim \N(\bar{x}_i, \alpha_i^2), \bar{x}_i \sim \N\Par*{\mu,
  \frac{\sigma^2}{n}}, \mu \sim \N\Par*{0, \frac{1}{\tau}}.
\]
Denote $\beta_k \defeq \frac{1}{\frac{\sigma^2}{n} + \alpha_k^2}$ and
$\rho \defeq \frac{n}{\sigma^2}$. We get the following density function
\begin{align*}
  p & \propto \exp\Par[\bigg]{-\frac{(m_i - \bar{x}_i)^2}{2 \alpha_i^2} -
  \sum_{k \neq i} \frac{\beta_k (m_k - \mu)^2}{2} - \frac{\rho (\bar{x}_i -
  \mu)^2}{2} - \frac{\tau \mu^2}{2}} \\
  & \propto \exp\Par[\bigg]{-\frac{(m_i - \bar{x}_i)^2}{2 \alpha_i^2} -
  \frac{\sum_{k \neq i} \beta_k m_k^2}{2} - \frac{\rho \bar{x}_i^2}{2} +
  \Par[\Big]{\sum_{k \neq i} \beta_k m_k + \rho \bar{x}_i} \mu
  - \frac{(\sum_{k \neq i} \beta_k + \rho + \tau) \mu^2}{2}}.
\end{align*}
As we can see, the only statistics we use from $\{m_k\}_{k \neq i}$ is $s_i
\defeq \sum_{k \neq i} \beta_k m_k \sim \N(\gamma_i \mu, \gamma_i)$,
where $\gamma_i \defeq \sum_{k \neq i} \beta_k$. Using this variable,
we get
\[
  p \propto \exp\Par*{-\frac{(m_i - \bar{x}_i)^2}{2 \alpha_i^2} -
  \frac{s_i^2}{2 \gamma_i} - \frac{\rho \bar{x}_i^2}{2} + (s_i + \rho
  \bar{x}_i) \mu - \frac{(\gamma_i + \rho + \tau) \mu^2}{2}}.
\]
After we integrate over $\mu$, we get
\begin{align*}
  p & \propto \exp\Par*{-\frac{(m_i - \bar{x}_i)^2}{2 \alpha_i^2} -
  \frac{s_i^2}{2 \gamma_i} - \frac{\rho \bar{x}_i^2}{2} + \frac{(s_i +
  \rho \bar{x}_i)^2}{2 (\gamma_i + \rho + \tau)}} \\
  & \propto \exp\Par*{-\frac{m_i^2}{2 \alpha_i^2} -
  \frac{s_i^2}{2 \gamma_i} + \Par*{\frac{m_i}{\alpha_i^2} + \frac{\rho
  s_i}{\gamma_i + \rho + \tau}} \bar{x}_i - \Par*{\frac{1}{\alpha_i^2} +
  \frac{\rho (\gamma_i + \tau)}{\gamma_i + \rho + \tau}}
  \frac{\bar{x}_i^2}{2}}.
\end{align*}
So, $\bar{x}_i$ is distributed as
\[
  \N\Par*{\Par*{\frac{1}{\alpha_i^2} + \frac{\rho (\gamma_i +
  \tau)}{\gamma_i + \rho + \tau}}^{-1} \Par*{\frac{m_i}{\alpha_i^2} +
  \frac{\rho s_i}{\gamma_i + \rho + \tau}}, \Par*{\frac{1}{\alpha_i^2}
  + \frac{\rho (\gamma_i + \tau)}{\gamma_i + \rho + \tau}}^{-1}}
\]
and the mean of this distribution is the optimal estimator for $\bar{x}_i$.

\subsection{Proof of \texorpdfstring{\cref{thm:b-mean-accuracy}}{Theorem
\ref{thm:b-mean-accuracy}}}
\label{sec:b-mean-accuracy-proof}

Similarly to the previous section we have the following distributions of
$\bar{x}_i$, $m_i$, and $\mu$
\[
  \forall k \neq i \: m_k \sim \N\Par*{\mu, \frac{\sigma^2}{n} + \alpha_k^2},
  \bar{x}_i \sim \N\Par*{\mu, \frac{\sigma^2}{n}}, \mu \sim \N\Par*{0,
  \frac{1}{\tau}}.
\]
It implies the following density
\[
  p \propto \exp\Par[\bigg]{- \sum_{k \neq i} \frac{\beta_k (m_k - \mu)^2}{2}
  - \frac{\rho (\bar{x}_i - \mu)^2}{2} - \frac{\tau \mu^2}{2}}
\]
Again, the sufficient statistics of $\{m_k\}_{k \neq i}$ is $s_i$ and we get
the following density
\begin{align*}
  p & \propto \exp\Par[\bigg]{- \frac{(s_i - \gamma_i \mu)^2}{2 \gamma_i}
  - \frac{\rho (\bar{x}_i - \mu)^2}{2} - \frac{\tau \mu^2}{2}} \\
  & \propto \exp\Par[\bigg]{- \frac{s_i^2}{2 \gamma_i}
  - \frac{\rho \bar{x}_i^2}{2} + (s_i + \rho \bar{x}_i) \mu -
  \frac{(\gamma_i + \rho + \tau) \mu^2}{2}}.
\end{align*}
Thus, $\mu$ is distributed as $\N\Par*{\frac{s_i + \rho \bar{x}_i}{\gamma_i
+ \rho + \tau}, \frac{1}{\gamma_i + \rho + \tau}}$ and the optimal
estimator is the mean of this distribution.

\subsection{Proof of \texorpdfstring{\cref{thm:b-mean-first-order}}{Theorem
\ref{thm:b-mean-first-order}}}
\label{sec:b-mean-first-order-proof}

The first order Taylor approximation gives the following inequality
\begin{multline*}
  \Par[\Big]{\frac{1}{\tau} + \frac{1}{\rho}} \Par[\Big]{1 -
  \frac{\beta_i}{\rho} + \frac{\gamma_i}{\rho + \tau} - \frac{\beta_i +
  \gamma_i}{\tau}} - \frac{\lambda_i}{\rho + \tau} \Par[\Big]{1 -
  \frac{\gamma_i}{\rho + \tau}} \ge \Par[\Big]{\frac{1}{\tau} + \frac{1}{\rho}}
  - \frac{\lambda_i}{\rho + \tau} \\
  \iff \Par*{\frac{\lambda_i}{(\rho + \tau)^2} - \frac{1}{\tau^2}} \gamma_i -
  \frac{(\rho + \tau)^2}{\rho^2 \tau^2} \beta_i \ge 0
  \iff
  \Par*{\frac{\lambda_i}{(\rho + \tau)^2} - \frac{1}{\tau^2}} \sum_{i=1}^N
  \beta_k - \Par*{\frac{\lambda_i}{(\rho + \tau)^2} + \frac{(2 \rho +
  \tau)}{\rho^2 \tau}} \beta_i \ge 0.
\end{multline*}

Notice that the inequality have a solution only if
\[
  \frac{\lambda_i}{(\rho + \tau)^2} - \frac{1}{\tau^2} \ge 0.
\]

Now, we assume that all clients care about utility strongly enough, so that
\[
  \forall i \: \frac{\lambda_i}{(\rho + \tau)^2} - \frac{1}{\tau^2} \ge 0,
\]
and derive necessary conditions for the feasibility of the participation of all
clients. Notice that
\[
  \Par[\Big]{\frac{\lambda_i}{(\rho + \tau)^2} - \frac{1}{\tau^2}} \sum_{i=1}^N
  \beta_k - \Par[\Big]{\frac{\lambda_i}{(\rho + \tau)^2} + \frac{(2 \rho +
  \tau)}{\rho^2 \tau}} \beta_i \ge 0
  \iff \xi_i \sum_{i=1}^N \beta_k - \beta_i \ge 0,
\]
where
\[
  \xi_i \defeq \frac{\frac{\lambda_i}{(\rho + \tau)^2} -
  \frac{1}{\tau^2}}{\frac{\lambda_i}{(\rho + \tau)^2} + \frac{(2 \rho +
  \tau)}{\rho^2 \tau}} = \frac{1}{1 + \frac{(\rho + \tau)^2}{\tau^2 \rho^2
  \Par*{\frac{\lambda_i}{(\rho + \tau)^2} - \frac{1}{\tau^2}}}}.
\]
By summing inequalities for each client, we get
\[
  \Par[\Bigg]{\sum_{j=1}^N \xi_j} \Par*{\sum_{k=1}^N \beta_k} \ge \sum_{k=1}^N
  \beta_k.
\]
So, the necessary condition for the existence of solution is
\[
  \sum_{j=1}^N \xi_j \ge 1.
\]
This condition is also sufficient: when it holds $\beta_i = \xi_i
b, b > 0$ is a solution.

\subsection{Proof of
\texorpdfstring{\cref{thm:b-mean-symmetric-general}}{Theorem
\ref{thm:b-mean-symmetric-general}}}
\label{sec:b-mean-symmetric-general-proof}

In this case, we have $\gamma_i = (N - 1) \beta$. The inequality will look like
\begin{multline*}
  \frac{(\rho - \beta) ((N - 1) \beta + \rho + \tau)}{(N \beta + \tau) \rho^2}
  - \frac{\lambda}{(N - 1) \beta + \rho + \tau} \ge \frac{\rho + \tau}{\rho
  \tau} - \frac{\lambda}{\rho + \tau} \\
  \iff - \frac{N - 1}{N \rho^2} \beta - \frac{(\rho + \frac{\tau}{N})^2}{\tau
  \rho^2} + \frac{\lambda}{\rho + \tau} + \frac{(\rho +
  \frac{\tau}{N})^2}{N \rho^2 (\beta + \frac{\tau}{N})} - \frac{\lambda}{(N -
  1) (\beta + \frac{\rho + \tau}{N - 1})} \ge 0.
\end{multline*}

Denote left hand part of the inequality as $h_\lambda(\beta)$. Notice that
$h_\lambda(0) = 0$. Thus, $\max_\beta h_\lambda(\beta) \ge 0$. Now, we will
consider the optimal $\beta^*$ for the utility of all participants. To do it,
we calculate the derivative of $h_\lambda$
\[
  h'_\lambda(\beta) = - \frac{N - 1}{N \rho^2} - \frac{(\rho +
  \frac{\tau}{N})^2}{N \rho^2 (\beta + \frac{\tau}{N})^2} + \frac{\lambda}{(N -
  1) (\beta + \frac{\rho + \tau}{N - 1})^2}.
\]
Notice that $h_\lambda$ and $h'_\lambda$ are increasing in $\lambda$.

To study this function, we will take the second derivative
\begin{multline*}
  h''_\lambda(\beta) = 2 \frac{(\rho + \frac{\tau}{N})^2}{N
  \rho^2 (\beta + \frac{\tau}{N})^3} - 2 \frac{\lambda}{(N - 1) (\beta +
  \frac{\rho + \tau}{N - 1})^3} \\
  = \frac{2}{(\beta + \frac{\rho + \tau}{N - 1})^3}
  \Par[\Bigg]{\frac{(\rho + \frac{\tau}{N})^2}{N \rho^2} \Par*{1 + \frac{N \rho
  + \tau}{N (N - 1) (\beta + \frac{\tau}{N})}}^3 - \frac{\lambda}{N - 1}}.
\end{multline*}
As we can see, $h''_\lambda$ can change sign only once from positive to
negative on the interval $[0, \rho]$. Thus, $h'_\lambda$ can change sign only
twice: from negative to positive on the interval where $h''_\lambda > 0$ and
from positive to negative on the interval where $h''_\lambda < 0$.
\begin{enumerate}
  \item If $h''_\lambda$ is always positive $\lambda \le \frac{(N \rho +
    \tau)^2}{(N - 1)^2 \rho^2}$, the minimum of
    $h'_\lambda$ is $h'_\lambda(0)$ and the maximum is $h'_\lambda(\rho)$.
  \item If $h''_\lambda$ changes sign $\frac{(N \rho + \tau)^2 (\rho +
    \tau)^3}{(N - 1)^2 \rho^2 \tau^3} > \lambda > \frac{(N \rho + \tau)^2}{(N -
    1)^2 \rho^2}$, then one of the points $h'_\lambda(0)$ or $h'_\lambda(\rho)$
    is the minimum of $h'_\lambda$.
  \item If $h''_\lambda$ is always negative $\lambda \ge \frac{(N \rho +
    \tau)^2 (\rho + \tau)^3}{(N - 1)^2 \rho^2 \tau^3}$, the minimum of
    $h'_\lambda$ is $h'_\lambda(\rho)$ and the maximum is $h'_\lambda(0)$.
\end{enumerate}

Notice that in the first case
\[
  h'_\lambda(\rho) = -\frac{1}{\rho^2} + \frac{\lambda}{(N - 1) \Par*{\frac{N
  \rho + \tau}{N - 1}}^2} < -\frac{N - 2}{(N - 1) \rho^2} < 0.
\]
Therefore, $\beta^* = 0$. Since we are not interested in this trivial
solutions, we will consider only $\lambda > \frac{(N \rho + \tau)^2}{(N - 1)^2
\rho^2}$.

The previous analysis of the second derivative immediately gives the following
lemma

\begin{lemma}
  \label{lem:b-mean-cases}
  We can describe the behavior of $\beta^*$ using the following five cases
  \begin{enumerate}
    \item $h'_\lambda(0) > 0$ and $h'_\lambda(\rho) > 0$. Then $h'_\lambda$
      is always positive and $\beta^* = \rho$.

    \item $h'_\lambda(0) \le 0$ and $h'_\lambda(\rho) > 0$. Then $h'_\lambda$
      changes sign only once (from negative to positive) and $\beta^* \in \{0,
      \rho\}$.

    \item $h'_\lambda(0) > 0$ and $h'_\lambda(\rho) \le 0$. Then $h'_\lambda$
      changes sign only once (from positive to negative) and $\beta^* \in
      (h'_\lambda)^{-1}(0) \cap (h''_\lambda)^{-1}((-\infty, 0])$.

    \item $h'_\lambda(0) \le 0$, $h'_\lambda(\rho) \le 0$, and
      $h'_\lambda(\beta) = 0$ has roots. Then $h'_\lambda$ changes sign twice
      and $\beta^* \in \{0\} \cup ((h'_\lambda)^{-1}(0) \cap
      (h''_\lambda)^{-1}((-\infty, 0]))$.

    \item $h'_\lambda(0) \le 0$, $h'_\lambda(\rho) \le 0$, and
      $h'_\lambda(\beta) = 0$ does not have roots. Then $h'_\lambda$ is
      negative and $\beta^* = 0$.
  \end{enumerate}
\end{lemma}

Notice that when $h''_\lambda$ is always negative, case 4 can not be realized.
To distinguish between cases 4 and 5 when $h''_\lambda$ changes sign, one could
find the root of equation $h''_\lambda(\beta_\text{test}) = 0$ and look at the
sign of $h'_\lambda(\beta_\text{test})$. If $h'_\lambda(\beta_\text{test}) <
0$, then case 5 is realized. Otherwise, case 4 is realized.

We can describe cases 1--5 in the following manner. In case 1, when $\lambda$
is very big, the participants care a lot about accuracy and are ready to fully
compromise their privacy to get a good machine learning model. In case 5, when
$\lambda$ is very small, the participants care a lot about privacy and are not
ready to collaborate at all due to this constraint. Cases 2--4 are intermediate
between these two extremes: people care about accuracy of the model and are
ready to collaborate, but want to have some privacy protection.

The conditions in the theorem ensures that case 3 will be realized. To show
this, we will analyze the restrictions for this case
\[
  \begin{cases}
    h'_\lambda(0) &= - \frac{N \rho^2 + 2 \rho \tau + \tau^2}{\rho^2 \tau^2} +
    \frac{\lambda (N - 1)}{(\rho + \tau)^2} \ge 0,\\
    h'_\lambda(\rho) &= -\frac{1}{\rho^2} + \frac{\lambda (N - 1)}{(N \rho +
    \tau)^2} < 0,
  \end{cases}
  \iff \frac{(N \rho^2 + 2 \rho \tau + \tau^2) (\rho + \tau)^2}{(N - 1) \rho^2
  \tau^2} \le \lambda < \frac{(N \rho + \tau)^2}{(N - 1) \rho^2}.
\]
To make this case self-consistent, we should require
\[
  N \rho^3 + (2 N + 2) \rho^2 \tau \le (N^2 - N - 5) \rho \tau^2 + (2 N - 4)
  \tau^3,
\]
which holds when the number of participants is big enough.

In this case $h''_\lambda(\beta^*) \le 0$, $h'_\lambda(\beta^*) = 0$, and
$h_\lambda(\beta^*) > 0$. Using implicit function theorem, we get
\[
  0 = \odv{h'_\lambda}{\xi} = \pdv{h'_\lambda}{\xi} + h''_\lambda(\beta^*)
  \odv{\beta^*}{\xi} \implies \odv{\beta^*}{\xi} = -(h''_\lambda(\beta^*))^{-1}
  \pdv{h'_\lambda}{\xi},
\]
where $\xi$ is any parameter. Thus, the sign of $\odv{\beta^*}{\xi}$ is the
same as the sign of $\pdv{h'_\lambda}{\xi}$. It allows us to get the following
properties
\[
  \odv{\beta^*}{\lambda} \ge 0, \: \odv{\beta^*}{\rho} \ge 0.
\]

To prove these properties, we need to determine the signs of
$\pdv{h'_\lambda}{\lambda}$ and $\pdv{h'_\lambda}{\rho}$. For the first
derivative, it is straightforward
\[
  \pdv{h'_\lambda}{\lambda} = \frac{1}{(N - 1) (\beta^* + \frac{\rho + \tau}{N
  - 1})^2} > 0.
\]
For the second, we will use that $h'_\lambda(\beta^*) = 0$ and
$h''_\lambda(\beta^*) \le 0$
\begin{multline*}
  \pdv{h'_\lambda}{\rho} = \frac{N - 1}{N \rho^3} + \frac{2 (\rho +
  \frac{\tau}{N}) \frac{\tau}{N}}{N \rho^3 (\beta + \frac{\tau}{N})^2} -
  \frac{2 \lambda}{(N - 1)^2 (\beta + \frac{\rho + \tau}{N - 1})^3} \\
  = \frac{\lambda}{(N - 1) \rho (\beta + \frac{\rho + \tau}{N - 1})^2} -
  \frac{(\rho + \frac{\tau}{N})^2}{N \rho^3 (\beta + \frac{\tau}{N})^2} +
  \frac{2 (\rho + \frac{\tau}{N}) \frac{\tau}{N}}{N \rho^3 (\beta +
  \frac{\tau}{N})^2} - \frac{2 \lambda}{(N - 1)^2 (\beta + \frac{\rho + \tau}{N
  - 1})^3} \\
  = \frac{\lambda (\beta + \frac{\tau - \rho}{N - 1})}{(N - 1)^2 \rho (\beta +
  \frac{\rho + \tau}{N - 1})^3} - \frac{\rho^2 - \frac{\tau^2}{N^2}}{N \rho^3
  (\beta + \frac{\tau}{N})^2}
  \ge \frac{(\rho + \frac{\tau}{N})^2 (\beta + \frac{\tau - \rho}{N - 1})}{N
  \rho^3 (\beta + \frac{\tau}{N})^3} -\frac{\rho^2 - \frac{\tau^2}{N^2}}{N
  \rho^3 (\beta + \frac{\tau}{N})^2} \\
  = \frac{(\rho + \frac{\tau}{N})^2 (\beta + \frac{\tau - \rho}{N - 1}) -
  (\rho^2 - \frac{\tau^2}{N^2}) (\beta + \frac{\tau}{N})}{N \rho^3 (\beta +
  \frac{\tau}{N})^3} \ge 0,
\end{multline*}
where we have used that $(N - 1) \beta + \tau - \rho \ge 0$, which follows from
the lemma below.

\begin{lemma}
  \label{lem:b-mean-beta-ineq}
  Assume that $h'_\lambda(\beta) = 0$ and $h''_\lambda(\beta) < 0$. Then $(N -
  1) \beta + \tau \ge \sqrt[3]{N - 1} \rho$.
\end{lemma}

To proof this, we notice the following
\begin{multline*}
  \frac{(\rho + \frac{\tau}{N})^2}{N \rho^2 (\beta + \frac{\tau}{N})^3} \le
  \frac{\lambda}{(N - 1) (\beta + \frac{\rho + \tau}{N - 1})^3} 
  \iff \frac{(\rho + \frac{\tau}{N})^2}{N \rho^2 (\beta + \frac{\tau}{N})^3}
  \le \frac{1}{(\beta + \frac{\rho + \tau}{N - 1})} \Par*{\frac{N - 1}{N
  \rho^2} + \frac{(\rho + \frac{\tau}{N})^2}{N \rho^2 (\beta +
  \frac{\tau}{N})^2}} \\
  \iff \Par*{\rho + \frac{\tau}{N}}^2 \Par[\Big]{\beta + \frac{\rho + \tau}{N
  - 1}} \le (N - 1) \Par*{\beta + \frac{\tau}{N}}^3 + \Par*{\rho +
  \frac{\tau}{N}}^2 \Par*{\beta + \frac{\tau}{N}} \\
  \iff \Par*{\rho + \frac{\tau}{N}}^3 \le (N - 1)^2 \Par*{\beta +
  \frac{\tau}{N}}^3 
  \iff \sqrt[3]{N - 1} \Par*{\rho + \frac{\tau}{N}} \le (N -
  1) \beta + \tau - \frac{\tau}{N} \implies \rho \le (N - 1) \beta + \tau.
\end{multline*}

\subsection{Proof of \texorpdfstring{\cref{thm:b-mean-limits}}{Theorem
\ref{thm:b-mean-limits}}}
\label{sec:b-mean-limits-proof}

\subsubsection{\texorpdfstring{$N \to \infty$}{N \textbackslash to
\textbackslash infty}}

In this limit, $h'_\lambda(\rho) \to - \frac{1}{\rho^2} < 0$. According to
\cref{lem:b-mean-cases}, $\beta^* \in \{0\} \cup ((h'_\lambda)^{-1}(0) \cap
(h''_\lambda)^{-1}((-\infty, 0]))$.

We start with the possible solution $h'_\lambda(\beta) = 0$ and
$h''_\lambda(\beta) < 0$. \cref{lem:b-mean-beta-ineq} implies that $(N - 1)
\beta + \tau \ge \sqrt[3]{N - 1} \rho$. Thus, $\beta =
\Omega\Par*{\frac{1}{\sqrt[3]{N^2}}}$. It allows to write the leading terms
in equation $h'_\lambda(\beta) = 0$ in the following manner
\[
  h'_\lambda(\beta) = -\frac{1}{\rho^2} - \frac{1}{N \beta^2} +
  \frac{\lambda}{N \beta^2} + \mathrm{O}\Par*{\frac{1}{\sqrt[3]{N}}}.
\]
It gives $\beta = \sqrt{\frac{\lambda - 1}{N}} \rho +
\mathrm{o}\Par*{\frac{1}{\sqrt{N}}}$. It gives the following expression for the
final gain in utility
\[
  h_\lambda(\beta) = \frac{(\rho - \beta) ((N - 1) \beta + \rho + \tau)}{\rho^2
  (N \beta + \tau)} - \frac{\lambda}{(N - 1) \beta + \rho + \tau} -
  \frac{1}{\rho} - \frac{1}{\tau} + \frac{\lambda}{\rho + \tau} 
  = \frac{\lambda}{\rho + \tau} - \frac{1}{\tau} +
  \mathrm{O}\Par*{\frac{1}{\sqrt{N}}}.
\]
Therefore, in the case $\lambda < \frac{\rho + \tau}{\tau}$, it is optimal to
not collaborate, $\beta^* = 0$. And, in the case, $\lambda > \frac{\rho +
\tau}{\tau}$, it is optimal to collaborate, $\beta^* = \sqrt{\frac{\lambda -
1}{N}} \rho + \mathrm{o}\Par*{\frac{1}{\sqrt{N}}}$.

In the case $\lambda > \frac{\rho + \tau}{\tau}$, in the first approximation,
$\beta^*$ is increasing in $\lambda$ and $\rho$ and decreasing in $N$.
Moreover,
\[
  (\alpha^*)^2 = \frac{1}{\beta^*} - \frac{1}{\rho} \approx
  \frac{\sqrt{\frac{N}{\lambda - 1}} - 1}{\rho}
\]
is decreasing in $\rho$.

\subsubsection{\texorpdfstring{$\rho \to \infty$}{\textbackslash rho
\textbackslash to \textbackslash infty}}

According to \cref{lem:b-mean-cases}, $\beta^* \in \{0, \rho\} \cup
((h'_\lambda)^{-1}(0) \cap (h''_\lambda)^{-1}((-\infty, 0]))$.

First, we consider the possible solution $h'_\lambda(\beta) = 0$ and
$h''_\lambda(\beta) < 0$. \cref{lem:b-mean-beta-ineq} ensures that $\beta \ge
\frac{\rho - \tau}{N - 1}$. In this case,
\begin{multline*}
  h_\lambda(\beta) = \frac{(\rho - \beta)((N - 1) \beta + \rho +
  \tau)}{(N \beta + \tau) \rho^2} - \frac{\lambda}{(N - 1) \beta + \rho +
  \tau} - \frac{1}{\rho} - \frac{1}{\tau} + \frac{\lambda}{\rho + \tau} \\
  = \frac{(\rho - \beta)((N - 1) \beta + \rho)}{N \beta \rho^2} -
  \frac{\lambda}{(N - 1) \beta + \rho} - \frac{1}{\rho} - \frac{1}{\tau} +
  \frac{\lambda}{\rho} + \mathrm{o}\Par*{\frac{1}{\rho}}
  = -\frac{1}{\tau} + \mathrm{O}\Par*{\frac{1}{\rho}} < 0.
\end{multline*}
Similarly, $h_\lambda(\rho) = -\frac{1}{\tau} + \mathrm{O}\Par*{\frac{1}{\rho}}
< 0$. Therefore, $\beta^* = 0$.

\subsubsection{\texorpdfstring{$\rho \to 0$}{\textbackslash rho \textbackslash
to \textbackslash 0}}

According to \cref{thm:b-mean-symmetric-general}, the following condition on
$\lambda$ should hold for the existence of solutions
\[
  \lambda \ge \frac{(N \rho + \tau)^2}{(N - 1)^2 \rho^2}.
\]
However, the right-hand part goes to infinity when $\rho \to 0$. Thus, this
condition is not fulfilled for the small $\rho$.

\subsection{Proof of Theorem \ref{thm:existence_mean}}
\label{sec:existence_mean-proof}

Finally, we prove Theorem \ref{thm:existence_mean} which concerns the existence of mutually-beneficial protocols in the limit case when the number of participants $N\to \infty$.

\paragraph{Proof} Let $U = \{\tilde{u}_1, \ldots, \tilde{u}_{|U|}\}$. For any utility function $\tilde{u}^j$, since $\err^0 > 0$, we have that $\tilde{u}^j\left(\err^0, 0\right) < u^i(0,0)$. Since the function $\tilde{u}^j(0, \epsilon) \in U$ is continuous in $\epsilon$, there exists a value $\epsilon_{j} > 0$, such that $\tilde{u}^j\left(0, \epsilon_{j}\right) \in \left(\tilde{u}^j\left(\err^0, 0\right), \tilde{u}^j\left(0,0\right)\right]$. Let $\tilde{\epsilon} \coloneqq \min_{j} \epsilon_{j}$, so that $\tilde{u}^j\left(0, \tilde{\epsilon}\right)\in\left(\tilde{u}^j\left(\err^0, 0\right), \tilde{u}^j\left(0,0\right)\right]$ for all $j\in [|U|]$.

Since $\leak$ is decreasing in $\alpha$, there exists a value $\alpha_*$, such that $\leak \leq \tilde{\epsilon}$ for all participants. Therefore, for any $i \in [N],j \in [|U|]$, $\tilde{u}^j\left(\err_i^0, 0\right) < \tilde{u}^i\left(0, \tilde{\epsilon}\right) \leq \tilde{u}^i\left(0, \leak\right)$. Now for any $i \in [N], j \in [|U|]$ we have that $\tilde{u}^j(a, \leak)$ is a continuous function in $a$, so there exists a value $a_{i,j}$, such that $\tilde{u}^j(a_{i,j}, \leak) > \tilde{u}^j\left(\err_i^0, 0\right)$. Setting $\tilde{a} = \min_{i,j} a_{i,j}$, we have that $\tilde{u}^j(\tilde{a}, \leak) > \tilde{u}^j\left(\err_i^0, 0\right)$ for all $i \in [N], j\in [|U|]$. 

Now setting $N$ large enough so that $\err \leq \tilde{a}$ for all $i\in [N]$ ensures that $\tilde{u}^j(\err, \leak) \geq \tilde{u}^j(\tilde{a}, \leak) > \tilde{u}^j\left(\err_i^0, 0\right)$ for all $j\in [|U|]$ and any participant.
Therefore, for any set of utility functions from $U$ for the clients, we have that the protocol is mutually beneficial.

\section{ANALYSIS OF DIFFERENT UTILITY FUNCTIONS}
\label{sec:other_utilities}
Here we present additional analysis of our problem, in the case when the
utility of the participants is given by:
\[
  u_i(\err, \leak) = - \leak_i - \lambda_i \err^2_i.
\]
A rationnelle for this utility function is that the $\epsilon$ in DP privacy,
can be thought about as the number of bits that the algorithm's outputs reveil
about its input. Therefore, the loss in utility from a large $\epsilon$ may be
linear in $\epsilon$. At the same time, in our framework, $\err$ corresponds to
the root of the mean squared error of our estimate/end iterate. Therefore,
$\err^2$ corresponds to MSE more commonly used to quantify the loss of a
statistical/ML procedure.

We present an analysis for mean estimation and SGD with DP used as a privacy
notion.

\subsection{Mean estimation with DP}

We consider the mean estimation task with DP. We have $N$ participants who sample $n$ data points each from a global distribution $\mathcal{D}$ on $\mathbb{R}$, with mean $\mu$, variance $\mathbb{E}_{X\sim \mathcal{D}}\left((X - \mu)^2\right) = \sigma^2$ and bounded (almost surely) in $\mu +
  \Br*{-B / 2, B / 2}$. Each participant then communicates a message $m_i = \bar{x}_i + \epsilon_i$, where $\epsilon_i \sim \N(0, \alpha_i^2)$. For simplicity we do not consider Bayesian-optimal actions of the players based on all messages, but a simplified version similar to \cite{d23i}, in which the server communicates the mean of all messages $\bar{m} = \frac{1}{N}\sum m_i$ to all players. The players then accept this estimate, only correcting for their own added noice by taking $\bar{m}_i = \frac{1}{N}(N\bar{m} - m_i + \bar{x}_i)$.

We study the accuracy of the end estimates $\bar{m}_i$ and the privacy of these protocols. For the accuracy, note that Theorem 4.1 in \cite{d23i} implies that:
$$\err_i^2 = \mathbb{E}((\bar{m}_i - \mu)^2) = \frac{\sigma^2}{Nn} + \frac{1}{N^2}\sum_{j\neq i} \alpha_j^2,$$
compared to $\err_i^2 = \frac{\sigma^2}{n}$ when training a model alone. 
 
As shown in the main text, the privacy guarantee for the Gaussian mechanism implies that 
$$
  \leak_i = \frac{\sqrt{2 \ln(1.25 n^2)} B}{n\alpha_i}.
$$
We compare this to $\leak_i = 0$ when training locally.

Therefore, the condition for player $i$ being incentivized to join are as follows: 
\begin{align*}
\forall i \quad u_i \geq u_i^0 & \iff - \leak_i - \lambda_i \err^2_i \geq - \lambda_i \err^2_0 \\
& \iff - \frac{\sqrt{2 \ln(1.25 n^2)} B}{n\alpha_i} - \lambda_i \left(\frac{\sigma^2}{Nn} + \frac{1}{N^2}\sum_{j\neq i} \alpha_j^2\right) \geq - \lambda_i \frac{\sigma^2}{n}\\
& \iff \lambda_i \left(\frac{\sigma^2}{n} - \frac{\sigma^2}{Nn} - \frac{1}{N^2}\sum_{j\neq i} \alpha_j^2\right) \geq \frac{\sqrt{2 \ln(1.25 n^2)} B}{n\alpha_i}.
\end{align*}
Consider values of $\lambda_i$ bounded in some compact domain $\in [c, C]$ for some small constant $c$ and large constant $C$. As predicted by Theorem \ref{thm:existence_mean}, as $N\to \infty$, a mutually beneficial assignment of $\vec{\alpha}$ exists. In contrast, if $n \to \infty$, the left-hand side of the inequality becomes negative, so that no mutually beneficial alocation of $\vec{\alpha}$ exists.

\subsection{SGD with DP}

We consider the same setup as in \cref{sec:sgd_existence}, with batch size $b =
1$ and equal weights $a_i = \frac{1}{N}$. However, we make the following
assumption on the variance:
\begin{equation}
  \label{eqn:bounded_variance}
  \Var_{x}(\g_i(\w, x)) \le M + M_V \norm{\nabla f(\w)}^2 \: \forall w
  \in W.
\end{equation}
Additionally, we assume that $f(w) - f^*$ is bounded over $W$.

\subsubsection{Accuracy guarantees}
Given the assumption above, we apply the following improved version of Lemma
E.1 from \cite{d23i}.
\begin{lemma}[Adapted from Lemma E.1 in \citealt{d23i}]
  Consider $\eta \in \mathbb{N}$, such that $\frac{4}{\mu (\eta + 1)} \le
  \frac{1}{L (\nicefrac{M_V}{N} + 1)}$, the learning rate schedule $\eta^t =
  \frac{4}{\mu (\eta + t)}$. We get the following upper-bound for the
  optimization error:
  \begin{align*}
    \E(f(\w^t) - f(\w^*)) \le \frac{8 L \Par*{\frac{M}{N}+\frac{d
    \sum_{i=1}^N (\alpha_t^i)^2}{N^2}}}{3 \mu^2 t} +
    \O\Par*{\frac{1}{t^4} + \frac{1}{Nt}},
  \end{align*}
  as long as $\Pr(\exists t \le T: \varPi_W(\w^{t-1} - \eta^t \g^t) \ne
  \w^{t-1} - \eta^t \g^t) = \O\Par*{\frac{1}{Nt}}$.
\end{lemma}

\begin{proof}
  The proof follows the arguments in \cite{d23i}, however the dependence of the
  bound on the $\alpha$ parameters is traced explicitly and we improve the
  $\O\Par*{\frac{1}{t^2}}$ dependence to $\O\Par*{\frac{1}{t^4}}$ by carefully
  adapting a classic result from \cite{c54s}. 

  First assume that there is no $t \le T$, such that $\varPi_W(\w^{t-1} -
  \eta^t \g^t) \neq \w^{t-1} - \eta^t \g^t$. Since the stochastic gradients and
  the Gaussian noise are independent, we have:
  \begin{align*}
    \Var\Par*{\frac{1}{N}\sum_{i=1}^N m_i^t} &=
    \E\norm*{\frac{1}{N}\sum_{i=1}^N (g_i^t - \nabla f(\w^{t-1}) + \alpha_i
    \xi_i^t)}^2
    \\& \le \frac{1}{N^2} \Par*{\sum_{i=1}^N  M + M_V \norm{\nabla
    f(\w^{t-1})}^2 + d (\alpha_i)^2}
    \\& = \frac{M + M_V \norm{\nabla f(\w^{t-1})}^2}{N} + \frac{d \sum_{i=1}^N
    (\alpha_i)^2}{N^2}
	\end{align*}
	
	Applying equation 4.23 from \cite{b18o} with
  $\beta\xrightarrow{}\frac{4}{\mu}, \eta \xrightarrow{} \eta, \mu
  \xrightarrow{} 1, M_V \xrightarrow{} M_V/N, M \xrightarrow{} \frac{M}{N} +
  \frac{d\sum_{i=1}^N (\alpha_t^i)^2 }{N^2}, M_G \xrightarrow{} M_V/N + 1$, we
  get
	\begin{equation*}
    \E(f(\w^{t+1}) - f(\w^*)) \le \Par*{1 - \frac{4}{\eta + t}}
    \E\Par*{f(\w^t) - f(\w^*)} + \frac{8 L\left(\frac{M}{N}+\frac{d\sum_{i=1}^N
    (\alpha_i)^2 }{N^2}\right)}{\mu^2 (\eta + t)^2}
	\end{equation*}
	for any $t \ge 1$. To complete the proof, we will use the following version
  of a classic result by Chung:
  \begin{lemma}[Modified from \cite{c54s}]
    \label{lemma:better_chung}
    Let $\{b_n\}_{n\geq 1}$ be a sequence of real numbers, such that for some
    $n_0 \in \mathbb{N}$, it holds that for all $n \geq n_0$, 
    \begin{align*}
      b_{n+1} \leq \left(1 - \frac{c}{n}\right)b_n + \frac{c_1}{n^2},
    \end{align*}
    for some integer $c>1$ and some real-valued constant $c_1>0$. Then
    \begin{align*}
      b_n \leq \frac{c_1}{c - 1}\frac{1}{n} + \O\Par*{\frac{1}{n^c}}.
    \end{align*}
  \end{lemma}

  \begin{proof}
    Note that
    \begin{align*}
      \frac{c_1}{c-1}\left(\frac{1}{n+1} - \left(1 -
      \frac{c}{n}\right)\frac{1}{n}\right) & = \frac{c_1}{c-1}\frac{n^2 -
      n(n+1) + c(n+1)}{n^2 (n+1)}\\
      & = \frac{c_1}{c-1}\frac{cn - n + c}{n^2 (n+1)}\\
      & > \frac{c_1}{c-1}\frac{n(c - 1) + (c - 1)}{n^2 (n+1)}\\
      & \geq \frac{c_1}{n^2}
    \end{align*}
    Therefore, for any $n \ge n_0$, 
    \begin{align*}
      b_{n+1} \leq \left(1 - \frac{c}{n}\right)b_n +
      \frac{c_1}{c-1}\left(\frac{1}{n+1} - \left(1 -
      \frac{c}{n}\right)\frac{1}{n}\right),
    \end{align*}
    so that
    \begin{align*}
      b_{n+1} - \frac{c_1}{c-1}\frac{1}{n+1}\leq \left(1 -
      \frac{c}{n}\right)\left(b_n - \frac{c_1}{c-1}\frac{1}{n}\right),
    \end{align*}
    Denote by $b'_n = b_n - \frac{c_1}{c-1}\frac{1}{n}$, so that $b'_{n+1} \leq
    \left(1-\frac{c}{n}\right)b'_n$ for all $n \geq n_0$. If for some $n_1 >
    c$, $b'_{n_1} \leq 0$, then the same holds for any $n\geq n_1$, i.e.,
    \begin{align*}
      b_n \leq \frac{c_1}{c-1}\frac{1}{n},
    \end{align*}
    whenever $n\geq n_1$. Otherwise, it holds that for all $n \geq c + 1$,
    \begin{align*}
      0 < b'_n \leq b'_{c + 1} \prod_{m=c+1}^{n-1} \left(1 - \frac{c}{m}\right)
      =  b'_{c + 1} \frac{1}{\binom{n-1}{c}} =
      \O\left(\frac{1}{n^c}\right).
    \end{align*}
  \end{proof}

	Let $x_{t+\eta} \defeq \E\left(f(\w^t) - f^*\right)$ for $t \ge 1$ and $x_k
  \defeq \E\left(f(\w^0) - f^*\right)$ for $ k < \eta + 1$. Using
  the inequality above, we get that for all $k \ge \eta + 1$

  \begin{align*}
    x_{k+1} \le \Par*{1 - \frac{4}{k}} x_k + \frac{8 L
    \left(\frac{M}{N}+\frac{d\sum_{i=1}^N (\alpha_i)^2}{N^2}\right)}{\mu^2
    k^2} \implies
    x_t \le \frac{8 L \left(\frac{M}{N}+\frac{d\sum_{i=1}^N
    (\alpha_i)^2}{N^2}\right)}{3 \mu^2 t} + \O\left(\frac{1}{t^4}\right).
  \end{align*}
  Finally, the event that for some $t$ we have $\varPi_W(\w^{t-1} - \eta^t
  \g^t) \neq \w^{t-1} - \eta^t \g^t$ happens with probability
  $\O(\frac{1}{Nt})$ and the loss $f(\w) - f^*$ is bounded by a constant by
  assumption. Therefore, an additional term of $\O\Par*{\frac{1}{Nt}}$ appears,
  which completes the proof.
\end{proof} 

We would assume that the clients will use the following accuracy loss that
corresponds to our upper bound:
\begin{align*}
  \err_i^2 = \frac{8 L \left(\frac{M}{N}+\frac{d\sum_{i=1}^N
  (\alpha_i)^2}{N^2}\right)}{3 \mu^2 n} + \frac{C}{N n},
\end{align*}
where $C$ comes from the the $\O\Par*{\frac{1}{Nt}}$ in the theorem. The $\O\left(\frac{1}{t^4}\right)$-term is ignored as we consider a cross-silo setup, where $n$ is expected to be large.

\subsubsection{Privacy guarantees}
We use \cref{thm:dp-sgd-privacy}, leading to the following bound for the
privacy loss:
\begin{align*}
  \leak_i = \frac{16 \sqrt{2 \me \ln(1.25 n^3) \ln(4 n^2)} B}{\sqrt{n}
  \alpha_i}.
\end{align*}

\subsection{Utility Analysis}

We would get the following system of participation constraints,
\[
  u_i - u^0_i = -\frac{16 \sqrt{2 \me \ln(1.25 n^3) \ln(4 n^2)} B}{\sqrt{n}
  \alpha_i} + \lambda_i \Par*{\frac{8 L M + 3 \mu^2 C}{3 \mu^2 n} - \frac{8 L
  \left(\frac{M}{N}+\frac{d\sum_{i=1}^N (\alpha_i)^2}{N^2}\right)}{3 \mu^2 n} -
  \frac{C}{N t}}.
\]
As we can see, the privacy loss is of order $\frac{1}{\sqrt{n}}$ and the
accuracy benefit is bounded by $\O\Par*{\frac{1}{n}}$. This discrepancy
suggests that we should choose the $\alpha$ to be of order $\sqrt{n}$, but it
will eliminate the accuracy benefits if $N$ is bounded. Therefore, we do not
expect this system to have solutions in the limit $n \to \infty$, $N = \O(1)$.
On the other hand, in the limit $N \to \infty$, we can see that the
contribution of $\alpha_i$ disappears, which suggests that the system will have
solutions.

\end{document}